\documentclass[11pt]{article}
\usepackage[letterpaper]{geometry}

\usepackage[utf8]{inputenc} 
\usepackage[T1]{fontenc}    

\usepackage{microtype}
\usepackage{graphicx}
\usepackage{booktabs} 

\usepackage[colorlinks,citecolor=blue,urlcolor=blue,linkcolor=blue,linktocpage=true]{hyperref}

\usepackage{natbib}
\usepackage[utf8]{inputenc} 
\usepackage[T1]{fontenc}    
\usepackage[parfill]{parskip}
\usepackage{amsmath,amsthm,amssymb,bm,bbm}
\usepackage{mathtools}
\usepackage{cases}

\usepackage{microtype}

\usepackage{algorithm,algorithmic}
\usepackage{color}
\usepackage{appendix}
\usepackage{lmodern}

\usepackage{authblk}  

\usepackage{url}
\usepackage[colorlinks,citecolor=blue,urlcolor=blue,linkcolor=blue,linktocpage=true]{hyperref}

\usepackage{etoolbox}

\usepackage{cleveref}
\crefname{equation}{eq.}{eq.}
\Crefname{equation}{Eq.}{Eq.}

\crefname{definition}{\textbf{definition}}{definitions}
\Crefname{definition}{Definition}{Definitions}
\crefname{assumption}{\textbf{assumption}}{assumptions}
\Crefname{assumption}{Assumption}{Assumptions}

\definecolor{maroon}{RGB}{192,80,77}

\theoremstyle{plain}
\newtheorem{theorem}{Theorem}[section]

\newtheorem{lemma}[theorem]{Lemma}

\theoremstyle{definition}
\newtheorem{definition}[theorem]{Definition}
\newtheorem{assumption}[theorem]{Assumption}
\theoremstyle{remark}

\newcommand{\argmax}{\mathop{\mathrm{argmax}}}

\def\E{\mathbb{E}}
\def\P{\mathbb{P}}

\def\R{\mathbb{R}}
\def\cA{\mathcal{A}}

\def\cH{\mathcal{H}}

\def\cN{\mathcal{N}}

\usepackage{etoolbox}
\usepackage{comment}
\usepackage{caption}
\usepackage{subcaption}
\usepackage{isomath}
\usepackage{indentfirst}
\usepackage{enumerate}
\usepackage{enumitem}
\setlist{leftmargin=10mm}
\usepackage{mathrsfs}
\usepackage{multirow}
\usepackage{amstext}
\usepackage{dsfont}
\def\ind{\mathds{1}}

\newbool{twocol}
\setbool{twocol}{false}
\usepackage{graphicx}
\usepackage{diagbox}
\newcommand{\grad}{\nabla}

\newbool{compact}  
\setbool{compact}{true}
\usepackage{booktabs}       
\usepackage{amsfonts}       
\usepackage{nicefrac}       
\usepackage{xcolor}         

\title{Pricing with Contextual Elasticity and Heteroscedastic Valuation}

\author{Jianyu Xu}
\author{Yu-Xiang Wang}
\affil{Department of Computer Science\\
	University of California, Santa Barbara\\
\texttt{\{xu\_jy15, yuxiangw\}@ucsb.edu} }

\begin{document}

\maketitle
\begin{abstract}
	We study an online contextual dynamic pricing problem, where customers decide whether to purchase a product based on its features and price. 
We introduce a novel approach to modeling a customer's expected demand by incorporating feature-based price elasticity, which can be equivalently represented as a valuation with heteroscedastic noise.
To solve the problem, we propose a computationally efficient algorithm called "Pricing with Perturbation (PwP)", which enjoys an $O(\sqrt{dT\log T})$ regret while allowing arbitrary adversarial input context sequences. We also prove a matching lower bound at $\Omega(\sqrt{dT})$ to show the optimality regarding $d$ and $T$ (up to $\log T$ factors).
Our results shed light on the relationship between contextual elasticity and heteroscedastic valuation, providing insights for effective and practical pricing strategies.


\end{abstract}
\section{Introduction}
\label{sec:introduction}
Contextual pricing, a.k.a., Feature-based dynamic pricing, considers the problem of setting prices for a sequence of highly specialized or individualized products. With the growth of e-commerce and the increasing popularity of online retailers as well as customers, there has been a growing interest in this area \citep[see, e.g., ][]{amin2014repeated, qiang2016dynamic, javanmard2019dynamic, shah2019semi, cohen2020feature_journal, xu2021logarithmic, bu2022context}.

Formulated as a learning problem, the seller has no prior knowledge of ideal prices but is expected to learn on the fly by exploring different prices and adjusting their pricing strategy after collecting every demand feedback from customers.  Different from non-contextual dynamic pricing \citep{kleinberg2003value} where identical products are sold repeatedly, a contextual pricing agent is expected to generalize from one product to another in order to successfully price a previously-unseen product. A formal problem setup is described below:

\fbox{\parbox{0.97\textwidth}{Contextual pricing. For $t=1,2,...,T:$
		\small
		\noindent
		\begin{enumerate}[leftmargin=*,align=left]
			\setlength{\itemsep}{0pt}
			\item A product occurs, described by a context $x_t\in\R^{d}$.			
			\item The seller (we) proposes a price $p_t\geq0$.
			\item The customer reveals a demand $0\leq D_t\leq 1$.
			\item The seller gets a reward $r_t = p_t\cdot D_t$.
		\end{enumerate}
	}
}

Here $T$ is the time horizon, and the (random) demand $D_t$ is drawn from a distribution determined by context (or feature) $x_t$ and price $p_t$. The sequence of contexts $\{x_t\}$ can be either independently and identically distributed (iids) or chosen arbitrarily by an adversary. The seller's goal is to minimize the cumulative \emph{regret} against the sequence of optimal prices.

Existing works on contextual pricing usually assumes linearity on the demand, but they fall into two camps. On the one hand, the "linear demand" camp \citep{qiang2016dynamic, ban2021personalized, bu2022context} assumes the \emph{demand} $D_t$ as a (generalized) linear model. A typical model is $D_t = \lambda(\alpha p_t  + x_t^T \beta) + \epsilon_t$.  Here $\alpha<0$ is a parameter closely related to the \emph{price elasticity}. We will rigorously define a price elasticity in \Cref{appendix:definition} according to \citet{parkin2002economics}, where we also show that $\alpha$ is the \emph{coefficient of elasticity}. 
Besides of $\alpha$, other parameters like $\beta\in\R^d$ captures the base demand of products with feature $x_t$, $\epsilon_t$ is a zero-mean demand noise, and $\lambda$ is a known monotonically increasing link function. With this model, we have a noisy observation on the expected demand, which is reasonable as the same product is offered many times in period $t$. On the other hand, the "linear valuation" camp \citep{cohen2020feature_journal, javanmard2019dynamic, xu2021logarithmic} models a buyer's \emph{valuation} $y_t$ as linear and assumes a binary demand $D_t = \ind[p_t\leq y_t]$.  All three works listed above assume a \emph{linear-and-noisy} model with $y_t=x_t^\top\theta^*+N_t$, where $\theta^*\in\R^d$ is an unknown linear parameter that captures common valuations and $N_t$ is an idiosyncratic noise assumed to be iid.

Interestingly, the seemingly different modeling principles are closely connected to each other. In the "linear valuation" camp, notice that a customer's probability of "buying" a product equals $\E[D_t]$, which is further given by

$$
\E[D_t|p] = \P[y_t \geq p ] := S(p-x_t^\top\theta^*),
$$

where $S$ is the survival function of $N_t$ (i.e. $S(w)=1-\mathrm{CDF}(w)$ for $w\in\R$). This recovers a typical linear demand model by taking $\lambda(w)=S(-w)$ with $\alpha = -1$ and $\beta = \theta^*$. In other words, the distribution of $N_t$ completely characterizes the demand function $\lambda(\cdot)$ and vice versa.

However, the "linear demand" camp is not satisfied with a fixed $\alpha = -1$, while the "linear valuation" camp are skeptical about an observable demand $D_t$ even with zero-mean iid noise. One common limitation to both models is that neither captures how feature $x_t$ affects the price elasticity. 

\textbf{Our model.} To address this issue, we propose a natural model that unifies the perspectives of both groups. Also, we resolve the common limitation by modeling \emph{heteroscedasticity}, where we assume that the elasticity coefficient $\alpha$ is linearly dependent on feature $x_t$. This contextual modeling originates from the fact that different products have different price elasticities \citep{anderson1997price}.

In specific, we assume:

\begin{equation}
\label{equ:model}
D_t\sim \mathrm{Ber}(S(x_t^\top\eta^*\cdot p_t -x_t^\top\theta^*)),
\end{equation}
which adopts a generalized linear demand model (GLM) and a Boolean-censored feedback simultaneously. From the perspective of valuation model, it is \emph{equivalent} to assume
\begin{equation}
\label{equ:valuation_model}
D_t=\ind[p_t\leq y_t],\text{ where } y_t = \frac1{x_t^\top\eta^*}\cdot(x_t^\top\theta^* + N_t)\text{ and } \mathrm{CDF}_{N_t}(w)=1-S(w).
\end{equation}

Although \Cref{equ:model} seems more natural than \Cref{equ:valuation_model}, they are equivalent to each other (with reasonable assumptions on $S$). Notice that the random valuation $y_t$ is \emph{heteroscedastic}, which means its variance is not the same constant across a variety of $x_t$'s. We provide a detailed interpretation of this linear fractional valuation model in appendix.

\subsection{Contributions.} Our main results are twofold.
\begin{enumerate}[leftmargin=*]
	\item We propose a new demand model that assumes a feature-dependent price elasticity on every product. Equivalently, we model the heteroscedasticity on customers' valuations among different products. This model unifies the ``linear demand'' and ``linear valuation'' camps. 
	\item We propose a ``Pricing with Perturbation (PwP)'' algorithm that achieves $O(\sqrt{dT\log T})$ regret on this model, which is optimal up to $\log T$ factors. This regret upper bound holds for both i.i.d. and adversarial $\{x_t\}$ sequences.
\end{enumerate}

\subsection{Technical Novelty}\label{subsec:technical_novelty}
To the best of our knowledge, we are the first to study a contextual pricing problem with heteroscedastic valuation and Boolean-censored feedback. Some existing works, including \citet{javanmard2019dynamic, miao2019context, ban2021personalized, wang2021dynamic}, focus on related topics and achieve theoretical guarantees. However, their methodologies are not applicable to our settings due to substantial obstacles, which we propose novel techniques to overcome.

\textbf{Randomized surrogate regret}. \Citet{xu2021logarithmic} solves the problem with $x_t^\top\eta^*=1$, by taking the negative log-likelihood as a surrogate regret and running an optimization oracle that achieves a fast rate (i.e. an $O(\log T)$ regret). However, the log-likelihood is no longer a surrogate regret in our setting, since it is not "convex enough" and therefore cannot provide sufficient (Fisher) information. In this work, we overcome this challenge by constructing a \emph{randomized} surrogate loss function, whose \emph{expectation} is "strongly convex" enough to upper bound the regret. 

\textbf{OCO for adversarial inputs}. \citet{javanmard2019dynamic} and \citet{ban2021personalized} study the problem with unknown or heterogeneous noise variances (i.e. elasticity coefficients), but their techniques highly rely on the distribution of the feature distributions. As a result, their algorithm could be easily attacked by an adversarial $\{x_t\}$ series. In our work, we settle this issue by conducting an online convex optimization (OCO) scheme while updating parameters. Instead of estimating from the history that requires sufficient randomness in the inputs, our algorithm can still work well for adversarial inputs.

In addition, our algorithm has more advanced properties such as computational efficiency and information-theoretical optimality. For more highlights of our algorithm, please refer to \Cref{sec:algorithm}.


\section{Related Works}
\label{sec:related_works}
Here we present a review of the pertinent literature on contextual pricing and heteroscedasticity in machine learning, aiming to position our work within the context of related studies. For more related works on non-contextual pricing, contextual pricing, contextual searching and contextual bandits, please refer to \citet{wang2021multimodal}, \citet{xu2021logarithmic}, \citet{krishnamurthy2020contextual} and \citet{zhou2015survey} respectively.


\paragraph{Contextual Pricing.}
As we mentioned in \Cref{subsec:technical_novelty}, there are a large number of recent works on contextual dynamic pricing problems, and we refer to \citet{ban2021personalized} as a detailed introduction. On the one hand, \citet{qiang2016dynamic, nambiar2019dynamic, miao2019context, wang2021dynamic, ban2021personalized, bu2022context} assume a (generalized) linear demand model with noise, i.e. $\E[D_t] = g(\alpha p_t - \beta^\top x_t)$. Among those papers, \citet{miao2019context} worksl with a fixed $\alpha$ while we assume $\alpha$ as context-dependent.
\citet{wang2021dynamic} and \citet{ban2021personalized} are the closest to our problem setting, which consist of a generalized linear demand model and noisy observations. On the one hand, \citet{ban2021personalized} assumes independent add-on noises (while we allow binary martingale observations). With the development of a least-square estimator, they present an algorithm that achieves $\tilde{O}(s\sqrt{T})$ regret (with $s$ being the sparsity factor). On the other hand, \citet{wang2021dynamic} further gets rid of the independence among noises and allow them to be idiosyncratic. They proposes a UCB-based algorithm with $\tilde{O}(d\sqrt{T})$ regret and another Thompson-Sampling-based algorithm with $\tilde{O}(d^{\frac32}\sqrt{T})$ regret, both of which are sub-optimal in $d$. Moreover, all works mentioned above assume the context sequence $\{x_t\}$ to be i.i.d., whereas we consider it "too good to be true" and work towards an algorithm adaptive to adversarial input sequences. On the other hand, \citet{golrezaei2019incentive, shah2019semi, cohen2020feature_journal, javanmard2019dynamic, xu2021logarithmic, fan2021policy, goyal2021dynamic, luo2022contextual} adopts the linear valuation model $y_t=x_t^\top\theta^*+N_t$, which is a special case of our model as $x_t^\top\eta^*=1$. Specifically, 
both \citet{javanmard2019dynamic} and \citet{xu2021logarithmic} achieve an $O(d\log T)$ regret with $N_t$ drawn from a known distribution with $x_t^\top\eta^* = -1$. \citet{javanmard2019dynamic} also studies the setting when $x_t^\top\eta^*$ is fixed but unknown and achieves $O(d\sqrt{T})$ regret for stochastic $\{x_t\}$ sequences. In comparison, we achieve $O(\sqrt{dT\log T})$ on a more general problem and get rid of those assumptions. For a more detailed discussion, please refer to \Cref{table:comparing_existing_work}. 

\begin{table}[t]
        \small
	\label{table:comparing_existing_work}
	\begin{tabular}{|l|l|ll|ll|}
		\hline
		& \textbf{Known $\alpha$} & \multicolumn{2}{l|}{\textbf{Unknown fixed $\alpha$}} & \multicolumn{2}{l|}{\textbf{Heteroscedastic $\alpha=x_t^\top\eta^*$}} \\ \hline
		\textbf{Features}  &  \begin{tabular}[c]{@{}l@{}}Stochastic\\ \& Adversarial\end{tabular}                                    & \multicolumn{1}{l|}{Stochastic}                                                       & Adversarial                                                                       & \multicolumn{1}{l|}{Stochastic}                                                                                                                                                & Adversarial                                                                        \\ \hline
		\textbf{Upper Bound} & \begin{tabular}[c]{@{}l@{}} $d \log T$\\ {[}XW21{]} \end{tabular}     & \multicolumn{1}{l|}{\begin{tabular}[c]{@{}l@{}} $d \sqrt T$ \\ {[}JN19{]} \end{tabular}} & \textbf{\begin{tabular}[c]{@{}l@{}} $? \Rightarrow \sqrt{dT}$\\ This Work\end{tabular}} & \multicolumn{1}{l|}{\begin{tabular}[c]{@{}l@{}}$s \sqrt T $ (independent noises)\\ {[}BK21{]}\\ $d \sqrt T$  (idiosyncratic noises)\\ {[}WTL21{]} \end{tabular}} & \textbf{\begin{tabular}[c]{@{}l@{}}$? \Rightarrow \sqrt{dT}$ \\ This Work\end{tabular}} \\ \hline
		\textbf{Lower Bound} & \begin{tabular}[c]{@{}l@{}}$d\log{T}$\\ {[}BR12{]}\end{tabular} & \multicolumn{2}{l|}{\begin{tabular}[c]{@{}l@{}}$\sqrt T$\\ {[}JN19{]}\end{tabular}}                                                                                       & \multicolumn{2}{l|}{\textbf{\begin{tabular}[c]{@{}l@{}}$\sqrt{T} \Rightarrow \sqrt{dT}$\\ This Work\end{tabular}}}                                                                                                                                                         \\ \hline
	\end{tabular}
	\caption{Existing related literature and results on regret bounds, with $\tilde{O}(\cdot)$ dropped. Note that each adversarial setting covers the stochastic setting under the same assumptions. Here [XW21] stands for \citet{xu2021logarithmic}, [JN19] for \citet{javanmard2019dynamic}, [BR12] for \citet{broder2012dynamic}, [BK21] for \citet{ban2021personalized}, and [WTL21] for \citet{wang2021dynamic}.}
\end{table}



\paragraph{Heteroscedasticity.}
Since the valuation noise is scaled by a $\frac1{x_t^\top\eta^*}$ coefficient, the valuation is \emph{heteroscedastic}, referring to a situation where the variance is not the same constant across all observations. Heteroscedasticity may lead to bias estimates or loss of sample information. There are several existing methods handling this problem, including weighted least squares method \citep[][]{cunia1964weighted}, White's test \citep[][]{white1980heteroskedasticity} and Breusch-Pagan test \citep[][]{breusch1979simple}. Furthermore, \citet{anava2016heteroscedastic} and \citet{chaudhuri2017active} study online learning problems with heteroscedastic variances and provide regret bounds. For a formal and detailed introduction, we refer the audience to the textbook of \citet{kaufman2013heteroskedasticity}. 

\section{Problem Setup}
\label{sec:preliminaries}
\subsection{Notations}
To formulate the problem, we firstly introduce necessary notations and symbols used in the following sections. The sales session contains $T$ rounds with $T$ known to the seller in advance\footnote{Here we assume $T$ known for simplicity. For unknown $T$, we may apply a ``doubling epoch'' trick as \citet{javanmard2019dynamic} without affecting the regret rate.}. At each time $t=1,2,\ldots, T$, a product with feature $x_t\in\R^d$ occurs and we propose a price $p_t\geq 0$. Then the nature draws a demand $D_t\sim \mathrm{Ber}(S(x_t^\top\eta^*\cdot p_t -x_t^\top\theta^*))$, where $\theta^*, \eta^*\in\R^d$ are fixed unknown linear parameters and the link function $S: \R\rightarrow[0,1]$ is non-increasing. By the end of time $t$, we receive a reward $r_t = p_t\cdot D_t$. 

Equivalently, this customer has a valuation $y_t=\frac{x_t^\top\theta^* + N_t}{x_t^\top\eta^*}$ with noise $N_t\in\R$, and then make a decision $\ind_t = \ind[p_t\leq y_t]=D_t$ after seeing the price $p_t$. Similarly, we receive a reward $r_t = p_t\cdot\ind_t$. Assume $N_t\sim\mathbb{D}_F$ is independently and identically distributed (i.i.d.), with cumulative distribution function (CDF) $F=1-S$. Denote $s:=S'$ and $f:=F'$.

\subsection{Definitions}

Here we define some key quantities.
Firstly, we define an expected reward function.

\begin{definition}[expected reward function]
	\label{def:reward_function}
	Define
	\begin{equation}
	\label{equ:reward_function}
	\begin{aligned}
	r(u, \beta, p):=\E[r_t|x_t^\top\theta^*=u, x_t^\top\eta^*=\beta, p_t = p]= p\cdot S(\beta\cdot p-u)
	\end{aligned}
	\end{equation}
	as the expected reward function.
\end{definition}

Given this, we further define a greedy price function as the argmax of $r(u,\beta,p)$ over $p$.

\begin{definition}[greedy price function]
	\label{def:greedy_price}
	Define $J(u,\beta)$ as a greedy price function, i.e. the price that maximizes the expected reward given $u=x_t^\top\theta^*$ and $\beta=x_t^{\top}\eta^*$.
	\begin{equation}
	\label{equ:greedy_price}
	\begin{aligned}
	J(u,\beta) = \argmax_{p\in\R}r(u, \beta, p)=\argmax_{p\in\R}p\cdot S(\beta\cdot p - u)\\
	\end{aligned}
	\end{equation}
\end{definition}
Notice that
\begin{equation}
J(u, \beta)=\argmax_{p}p\cdot S(\beta p - u)= \frac1\beta\cdot\argmax_{\beta p}\beta p \cdot S(\beta p - u)= \frac1\beta J(u, 1).
\end{equation}
According to \citet[][Section B.1]{xu2021logarithmic}, we have the following properties.
\begin{lemma}
	\label{lemma: j_property}
	Denote $\varphi(w):=-\frac{S(w)}{s(w)}-w=\frac{1-F(w)}{f(w)}-w$, and we have $J(u, \beta)=\frac{u+\varphi^{-1}(u)}{\beta}$. Also, for $u\geq0$ and $\beta>0$, we have $\frac{\partial J(u, \beta)}{\partial u}\in(0,1)$. 
\end{lemma}

Then we define a negative log-likelihood function of parameter hypothesis ($\theta, \eta$) given the results at time $t$.

\begin{definition}[log-likelihood functions]
	\label{def:log_likelihood_neg}
	Denote $\ell_t(\theta, \eta)$ as the negative log-likelihood at time $t$, and define $L_t(\theta, \eta)$ as their summations:
	\begin{equation}
	\label{equ:log_likelihood_neg}
	\begin{aligned}
	-\ell_t(\theta, \eta)=&\ind_t\cdot\log S(x_t\top\eta\cdot p_t - x_t^\top\theta)+(1-\ind_t)\cdot\log(1-S(x_t^{\top}\eta\cdot p_t - x_t^\top\theta)).\\
	L_t(\theta, \eta)=&\sum_{\tau=1}^t \ell_t.
	\end{aligned}
	\end{equation}
\end{definition}
Finally, we define a round-$t$ expected regret and a cumulative expected regret.
\begin{definition}[regrets]
	Define $Reg_t(p_t):= r(x_t^\top\theta^*, x_t^\top\eta^*, J(x_t^\top\theta^*, x_t^\top\eta^*)) - r(x_t^\top\theta^*, x_t^\top\eta^*, p_t)$ as the expected regret at round $t$, conditioning on price $p_t$. Also, define the cumulative regret as follows

	\begin{equation}
	\label{equ:cumulative_regret}
	\begin{aligned}
	Regret = \sum_{t=1}^{T} Reg_t(p_t)
	\end{aligned}
	\end{equation}
\end{definition}
\subsection{Assumptions}
We establish three technical assumptions to make our analysis and presentation clearer. Firstly, we assume that all feature and parameter vectors are bounded within a unit ball in Euclidean norm. This assumption is without loss of generality as it only rescales the problem.

\begin{assumption}[bounded feature and parameter spaces]
	\label{assumption:bound}
	Assume features $x_t\in\cH_x$ and parameters $\theta\in\cH_{\theta}, \eta\in\cH_{\eta}$. Denote $U_p^d:=\{x\in\R^d, \|x\|_p\leq 1\}$ as an $L_p$-norm unit ball in $\R^d$. Assume all $\cH_x, \cH_{\theta}, \cH_{\eta}\in U_p^d$. Also, assume $x^\top\theta>0, \forall x\in\cH_x, \theta\in\cH_{\theta}$ and $x^\top\eta>C_{\beta}>0, \forall x\in\cH_x, \eta\in\cH_{\eta}$ for some constant $C_{\beta}\in(0,1)$.
\end{assumption}

The positiveness of elasticity coefficient $x^\top\eta>0$ comes from the \emph{Law of Demand} \citep{gale1955law, hildenbrand1983law}, stating that the quantity purchased varies inversely with price. This is derived from the \emph{Law of Diminishing Marginal Utilities} and has been widely accepted \citep{marshall2009principles}. We will show the necessity of assuming an elasticity lower bound $C_{\beta}$ in \Cref{appendix:more_discussion}. In specific, we claim that any algorithm will suffer a regret of $\Omega(\frac1{C_{\beta}})$. For the simplicity of notation, we denote $[\theta; \eta]:=[\theta^\top, \eta^\top]^\top\in\R^{2d}$ as the combination of $d$-dimension column vectors $\theta$ and $\eta$. Since we know that $x_t^\top\theta\in[0,1]$ and $x_t^\top\eta\in[C_{\beta}, 1]$, we have $J(x_t^\top\theta, x_t^\top\eta)\in[J(0,1), J(1, C_{\beta})]$. Later we will show that the price perturbation is no more than $\frac{J(0,1)}{10}$. Therefore, we may have the following assumption. 

\begin{assumption}[bounded prices]
	\label{assumption:bounded_price}
	For any price $p_t$ at each time $t=1,2,\ldots, T$, we require $p_t\in[c_1, c_2]$, where $c_1 = \frac{J(0,1)}2$ and $c_2 = 2J(1, C_\beta)$.
\end{assumption}

Similar to \citet{javanmard2019dynamic} and \citet{xu2021logarithmic}, we also assume a log-concavity on the noise CDF.

\begin{assumption}[log-concavity]
	\label{assumption:log-concave}
	Every $D_t$ is independently sampled according to \Cref{equ:model}, with $S(\omega)\in[0,1]$ and $s(\omega)=S'(\omega)>0, \forall \omega\in\R$. Equivalently, the valuation noise $N_t\sim\mathbb{D}_F$ is independently and identically distributed (i.i.d.), with CDF $F=1-S$. Assume that $S\in\mathbb{C}^2$, and $S$ and $(1-S)$ are strictly log-concave.
\end{assumption}

\section{Main Results}
\label{sec:results}
To solve the contextual pricing problem with featurized elasticity, we propose our ``Pricing with Perturbation (PwP)'' algorithm. In the following, we firstly describe the algorithm and highlight its properties, then analyze (and bound) its cumulative regret, and finally prove a regret lower bound to show its optimality.

\subsection{Algorithm}
\label{sec:algorithm}
The pseudocode of PwP is displayed as \Cref{algorithm:ONSPP}, which calls an ONS oracle (\Cref{algorithm: ONS}).

\begin{algorithm}[h]
	\caption{Pricing with Perturbation (PwP)}
	\label{algorithm:ONSPP}
	\begin{algorithmic}[1]
		\STATE {\bfseries Input:} {parameter spaces $\cH_{\theta}$, $\cH_{\eta}$, link function $S$, time horizon $T$, dimension $d$}
		\STATE {\bfseries Initialization:}{parameters $\theta_1\in\cH_{\theta}$, $\eta_1\in\cH_{\eta}$, price perturbation $\Delta$, cumulative likelihood $L_0 = 0$, matrix $A_0=\epsilon\cdot I_{2d}$ and parameter $\epsilon, \gamma$}
		\FOR{$t=1, 2, \ldots, T$}
		\STATE Observe $x_t$;
		\STATE Calculate greedy price $\hat{p}_t = J(x_t^\top\theta_t, x_t^\top\eta_t)$
		\STATE Sample $\Delta_t = \Delta$ with $\Pr=0.5$ and $\Delta_t = -\Delta$ with $\Pr = 0.5$;
		\STATE Propose price $p_t = \hat{p}_t + \Delta_t$;
		\STATE Receive the customer's decision $\ind_t$;
		\STATE Construct negative log-likelihood $\ell_t(\theta, \eta)$ and $L_t(\theta, \eta)$ as \cref{equ:log_likelihood_neg};
		\STATE Update parameters:
		\begin{equation*}
		[\theta_{t+1}; \eta_{t+1}]\leftarrow ONS([\theta_t;\eta_t])
		\end{equation*}
		\ENDFOR
	\end{algorithmic}
\end{algorithm}

\begin{algorithm}[t]
	\caption{Online Newton Step (ONS)}
	\label{algorithm: ONS}
	\begin{algorithmic}[1]
		\STATE {\bfseries Input: current parameter $[\theta_t, \eta_t]$, likelihood $\ell_t(\theta, \eta)$, matrix $A_t$, parameter $\gamma$, parameter spaces $\cH_{\theta}$ and $\cH_{\eta}$.}
		\STATE Calculate $\nabla_t=\grad \ell_t(\theta, \eta)$;
		\STATE Rank-1 update: $A_t = A_{t-1} + \nabla_t\nabla_t^{\top}$;
		\STATE Newton step: $[\hat\theta_{t+1}; \hat\eta_{t+1}] = [\hat\theta_t; \hat\eta_t] - \frac1{\gamma}A_t^{-1}\nabla_t$;
		\STATE Projection: $[\theta_{t+1}; \eta_{t+1}] = \Pi_{\cH_\theta\times\cH_\eta}^{A_t}([\hat\theta_{t+1}; \hat\eta_{t+1}])$;
	\end{algorithmic}
\end{algorithm}

At each time $t$, it inherits parameters $\theta_t$ and $\eta_t$ from $(t-1)$ and takes in a context vector $x_t$. By trusting in $\theta_t$ and $\eta_t$, it calculates a greedy price $\hat p_t$ and outputs a perturbed version $p_t = \hat p_t + \Delta_t$. After seeing customer's decision $\ind_t$, PwP calls an ``Online Newton Step (ONS)'' oracle (see \Cref{algorithm: ONS}) to update the parameters as $\theta_{t+1}$ and $\eta_{t+1}$ for future use.

\subsubsection{Highlights} We highlight the achievements of the PwP algorithm in the following three aspects.

\textbf{In this pricing problem.} As we mentioned in \Cref{subsec:technical_novelty}, the key to solving this contextual elasticity (or heteroscedastic valuation) pricing problem is to construct a surrogate loss function. \citet{xu2021logarithmic} adopts negative log-likelihood in their setting, which does not work for ours since it is not "convex" enough. In our PwP algorithm, we overcome this challenge by introducing a perturbation $\Delta$ on the proposed greedy price. This idea originates from the observation that the \emph{variance} of $p_t$ contributes positively to the "convexity" of the expected log-likelihood, which helps "re-build" the upper-bound inequality. 

\textbf{In online optimization.} PwP perturbs the greedy action (price) it should have taken. This idea is similar to a "Following the Perturbed Leader (FTPL)" algorithm \citep{hutter2005adaptive} that minimizes the summation of the empirical risk and a random loss function serving as a perturbation. However, this might lead to extra computational cost as the random perturbation is not necessarily smooth and therefore hard to optimize. In this work, PwP introduces a possible way to overcome this obstacle: Instead of perturbing the objective function, we may directly perturb the action to explore its neighborhood. Our regret analysis and results indicate the optimality of this method and imply a potentially wide application.

\textbf{In information theory.} We show the following fact in the regret analysis of PwP: By adding $\Delta$ perturbation on $p_t$, we may lose $O(\Delta^2)$ in reward but will gain $O(\Delta^2)\cdot I$ in Fisher information (i.e. the expected Hessian of negative log-likelihood function) in return. By Cramer-Rao Bound, this leads to $O(\frac1{\Delta^2})$ estimation error. In this way, we quantify the information (observing from exploration) on the scale of reward, which shares the same idea with the Upper Confidence Bound \citep{lai1985asymptotically} method that always maximizes the summation of empirical reward and information-traded reward.

Besides, PwP is computationally efficient as it only calls the ONS oracle for once. As for the ONS oracle, it updates an $A_t^{-1} = (A_{t-1} + \nabla_t\nabla_t^\top)^{-1}$ at each time $t$, which is with $O(d^2)$ time complexity according to the following \emph{Woodbury matrix identity}

\begin{equation}
\label{equ:woodbury}
(A+xx^\top)^{-1} = A^{-1} - \frac1{1+x^\top A^{-1}x}A^{-1}x(A^{-1}x)^\top.
\end{equation}

\subsection{Regret Upper Bound}
\label{sec:regret_analysis} Now we analyze the regret of PwP and propose an upper bound up to constant coefficients.
\begin{theorem}
	\label{thm: main_regret}
	Under \Cref{assumption:bound}, \ref{assumption:bounded_price} and \ref{assumption:log-concave}, by taking $\Delta = \min\left\{\left(\frac{d\log T}{T}\right)^{\frac14}, \frac{J(0,1)}{10}, \frac1{10}\right\}$, the algorithm PwP guarantees an expected regret at $O(\sqrt{dT\log T})$.
\end{theorem}
In the following, we prove \Cref{thm: main_regret} by stating a thread of key lemmas. We leave the detailed proof of those lemmas to \Cref{appendix: proof}.
\begin{proof}
	The proof overview can be displayed as the following roadmap of inequalities:
	\begin{equation}
	\label{equ:proof_roadmap}
	\begin{aligned}
	\E[Regret]
	=\sum_{t=1}^T Reg_t(p_t)\leq&\E\left[\sum_{t=1}^T O\left((x_t^\top(\theta_t-\theta^*))^2 + (x_t^\top(\eta_t-\eta^*))^2 + \Delta^2\right)\right]\\
	\leq&O\left(\frac{\sum_{t=1}^T \E\left[\ell_t(\theta_t, \eta_t) - \ell_t(\theta^*, \eta^*)\right]}{\Delta^2} + T\cdot\Delta^2\right)\\
	\leq&O\left(\frac{d\log T}{\Delta^2} + T\cdot\Delta^2\right)
	=O(\sqrt{dT\log T}).
	\end{aligned}
	\end{equation} 
	Here the first inequality is by the smoothness of regret function (see \Cref{lemma:smooth_regret}), the second inequality is by a special ``strong convexity'' of $\ell_t(\theta, \eta)$ that contributes to the surrogate loss (see \Cref{lemma:surrogate_expected_regret}), the third inequality is by Online Newton Step (see \Cref{lemma:onsp}), and the last equality is by the value of $\Delta$. A rigorous version of \Cref{equ:proof_roadmap} can be found in \Cref{subsec:proof_main_regret}.
	
	We firstly show the smoothness of $Reg_t(p_t)$:
	\begin{lemma}[regret smoothness]
		\label{lemma:smooth_regret}
		Denote $p_t^*:=J(x_t^\top\theta^*, x_t^\top\eta^*)$. There exists constants $C_r>0$ and $C_J>0$ such that
		\begin{equation}
		\label{equ:regret_smooth}
		\begin{aligned}
		Reg_t(p_t)
		\leq C_r\cdot (p_t-p_t^*)^2\leq C_r\cdot 2\left(C_J\cdot\left[(x_t^\top(\theta_t-\theta^*))^2 + (x_t^\top(\eta_t-\eta^*))^2\right] + \Delta^2\right).
		\end{aligned}
		\end{equation}
	\end{lemma}
	While the first inequality of \Cref{equ:regret_smooth} is from the smoothness, and the second inequality is by the Lipschitzness of function $J(u, \beta)$. Please refer to \Cref{subsec:proof_lemma_smooth_regret} for proof details.
	We then show the reason why the log-likelihood function can still be a surrogate loss with carefully randomized $p_t$.
	
	\begin{lemma}[surrogate expected regret]
		\label{lemma:surrogate_expected_regret}
		There exists a constant $C_l>0$ such that $\forall \theta\in\cH_\theta, \eta\in\cH_\eta$, we have
		\begin{equation}
		\label{equ:surrogate}
		\begin{aligned}
		&\E[\ell_t(\theta, \eta) - \ell_t(\theta^*, \eta^*)|\theta_t, \eta_t]\\
		\geq&\frac{C_l\Delta^2}{10}[(\theta-\theta^*)^\top, (\eta-\eta^*)^\top]
		\left[\begin{array}{cc}
		x_tx_t^\top & 0\\
		0 & x_tx_t^\top
		\end{array}
		\right]
		\left[\begin{array}{c}
		\theta-\theta^*\\
		\eta-\eta^*
		\end{array}
		\right]\\
		=&\frac{C_l\cdot \Delta^2}{10}\left[\left(x_t^\top(\theta-\theta^*)\right)^2 + \left(x_t^\top(\eta-\eta^*)\right)^2\right].
		\end{aligned}
		\end{equation}
	\end{lemma}
	This is the most important lemma in this work. We show a proof sketch here and defer the detailed proof to \Cref{subsec:proof_lemma_surrogate_expected_regret}.
	\begin{proof}[Proof sketch of \Cref{lemma:surrogate_expected_regret}]
		We show that there exist constants $C_l>0, C_p>0$ such that
		\begin{enumerate}
			\item $\nabla^2 \ell_t(\theta, \eta)\succeq C_l\cdot\left[
			\begin{array}{cc}
			x_tx_t^\top & -p_t\cdot x_tx_t^\top\\
			-p_t\cdot x_tx_t^\top & p_t^2\cdot x_tx_t^\top
			\end{array}
			\right]$, and
			\item $\E\left[
			\begin{array}{cc}
			x_tx_t^\top & -p_t\cdot x_tx_t^\top\\
			-p_t\cdot x_tx_t^\top & p_t^2\cdot x_tx_t^\top
			\end{array}
			|\theta_t, \eta_t\right]\succeq C_p\Delta^2 \left[
			\begin{array}{cc}
			x_tx_t^\top & 0\\
			0 & x_tx_t^\top
			\end{array}
			\right]$.
		\end{enumerate}
		The first property above relies on the exp-concavity of $\ell_t$. Notice that the second property does not hold without the $\E$ notation, as the left hand side is a $(a-b)^2$ form while the right hand side is in a $(a^2+b^2)$ form. In general, there exist no constant $c>0$ such that $(a-b)^2\geq c(a^2+b^2)$. However, due to the randomness of $p_t$, we have
		\begin{equation}
		\label{equ:var_pt}
		\E[p_t^2|\hat p_t] = \E[p_t|\hat p_t]^2 + \Delta^2.
		\end{equation}
		In this way, the \emph{conditional expectation} of the left hand side turns to $(a-b)^2 + \lambda\cdot b^2$ and we have
		\begin{equation}
		\label{equ:analog}
		\begin{aligned}
		(a-b)^2+\lambda b^2=(\frac1{\sqrt{1+\frac{\lambda}2}}\cdot a - \sqrt{1+\frac{\lambda}2}\cdot b)^2 + (1-\frac1{1+\frac{\lambda}2})a^2 + \frac{\lambda}2 b^2\geq\frac{\frac{\lambda}2}{1+\frac{\lambda}2}\cdot(a^2+b^2).
		\end{aligned}
		\end{equation}
		Similarly, we upper bound $\left[\begin{array}{cc}
		x_tx_t^\top & 0\\
		0 & x_tx_t^\top
		\end{array}\right]$ with $\E[\grad^2\ell_t(\theta, \eta)|\theta_t, \eta_t]$ up to a $C_p\cdot\Delta^2$ coefficient.
		
		With those two properties above, along with a property of likelihood function that $\E[\grad\ell_t(\theta^*, \eta^*)]=0$, we can prove \Cref{lemma:surrogate_expected_regret} by taking a Taylor expansion of $\ell_t$ at $[\theta^*;\eta^*]$.
	\end{proof}
	Finally, we cite a theorem from \citet{hazan2019introduction} as our \Cref{lemma:onsp} that reveals the surrogate regret rate on negative log-likelihood functions.
	\begin{lemma}
		\label{lemma:onsp}
		With parameters $G=\sup_{\theta\in\cH_{\theta}, \eta\in\cH_{\eta}}\|\nabla l_t(\theta, \eta)\|_2$, $D=\sup\|[\theta_1; \eta_1]-[\theta_2;\theta_2]\|\leq2$, $\alpha = C_e$, $\gamma=\frac12\min\{\frac1{4GD}, \alpha\}$ and $\epsilon=\frac1{\gamma^2D^2}$ and $T>4$, Keep running \Cref{algorithm: ONS} for $t=1,2,\ldots, T$ guarantees:
		\begin{equation}
		\label{equ:ons}
		\begin{aligned}
		\sup_{\{x_t\}}\left\{\sum_{t=1}^T\ell_t(\theta_t, \eta_t) - \min_{\theta\in\cH_{\theta}, \eta\in\cH_{\eta}}\sum_{t=1}^T\ell_t(\theta, \eta)\right\}\leq 5(\frac1{\alpha}+GD)d\log T.
		\end{aligned}
		\end{equation}
	\end{lemma}
	With all these lemma above, we have proved every line of \Cref{equ:proof_roadmap}.
\end{proof}

\subsection{Lower Bounds}
\label{subsec:lower_bound}
We claim that PwP is near-optimal in information theory, by proposing a matching regret lower bound in \Cref{thm:lower_bound}. We present the proof with valuation model to match with existing results.
\begin{theorem}
	\label{thm:lower_bound}
	Consider the contextual pricing problem setting with Bernoulli demand model given in \Cref{equ:model}. With all assumptions in \Cref{sec:preliminaries} hold, any pricing algorithm has to suffer a $\Omega(\sqrt{dT})$ worst-case regret, with $T$ the time horizon and $d$ the dimension of context.
\end{theorem}
\begin{proof}
	The main idea is to reduce $d$ numbers of 1-dimension problems to this problem setting. In fact, we may consider the following problem setting:
	\begin{enumerate}
		\item Construct set $X=\{x_i:=[0,\ldots, 0,1,0,\ldots,0]^\top\in\R^d\text{ with only }i^{\text{th}}\text{ place being 1}, i=1,2,\ldots, d\}$.
		\item Let $\theta^*=[\frac{u_1}{\sigma_1}, \frac{u_2}{\sigma_2}, \frac{u_3}{\sigma_3}, \ldots, \frac{u_d}{\sigma_d}]^\top, \eta^*=[\frac1{\sigma_1}, \frac1{\sigma_2}, \frac1{\sigma_3}, \ldots, \frac1{\sigma_d}]^\top$, and therefore we have $\frac{x_i^\top\theta^*+N_t}{x_i^\top\eta^*} = u_i + \sigma_i\cdot N_t$.
		\item At each time $t=1,2,\ldots, T$, sample $x_t\sim X$ independently and uniformly at random.
	\end{enumerate}
	In this way, we divide the whole time series $T$ into $d$ separated sub-problems, where the Sub-Problem $i$ has a valuation model $y_t(i) = u_i + \sigma_i\cdot N_t$, for $i=1,2,\ldots, d$.
 Let $N_t\sim\mathcal{N}(0,1), t=1,2,\ldots, T$, and $y_t(i)\sim\mathcal{N}(u_i, \sigma_i^2)$ are independent Gaussian random variables. 
 For each Sub-Problem $i$, it has a time horizon as $\frac Td$ in expectation. Let $u_i=\sqrt{\frac{\pi}2}$ and let each $\sigma_i^2$ be chosen from $\{1, 1-(\frac{T}{d})^{-\frac14}\}$. According to \citet[][Theorem 3.1]{broder2012dynamic} and \citet[][Theorem 12]{xu2021logarithmic}, the regret lower bound of each sub-problem is $\Omega(\sqrt{\frac Td})$. Therefore, the total regret lower bound is $d\cdot\Omega(\sqrt{\frac Td}) = \Omega(\sqrt{Td})$.
\end{proof}

\section{Numerical Experiments}
\label{sec:numerical_experiments}
\begin{figure*}[t]
	\centering
	\begin{subfigure}[t]{0.4\textwidth}
		\centering
		\includegraphics[width= \textwidth]{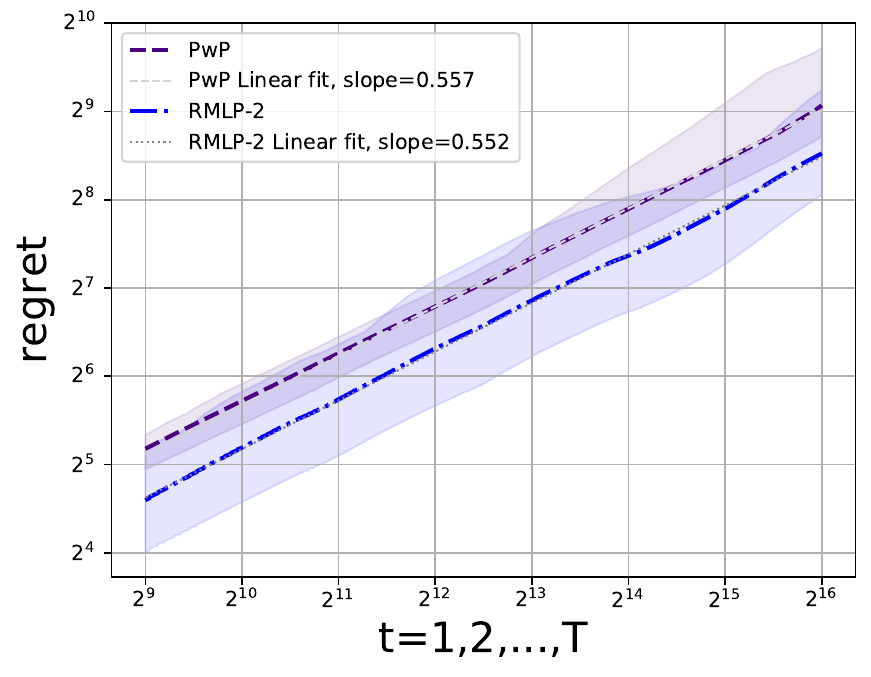}
		\caption{Stochastic $\{x_t\}$'s}\label{fig:stochastic_plot}
	\end{subfigure}
	\quad\quad
	\begin{subfigure}[t]{0.4\textwidth}
		\centering
		\includegraphics[width=\textwidth]{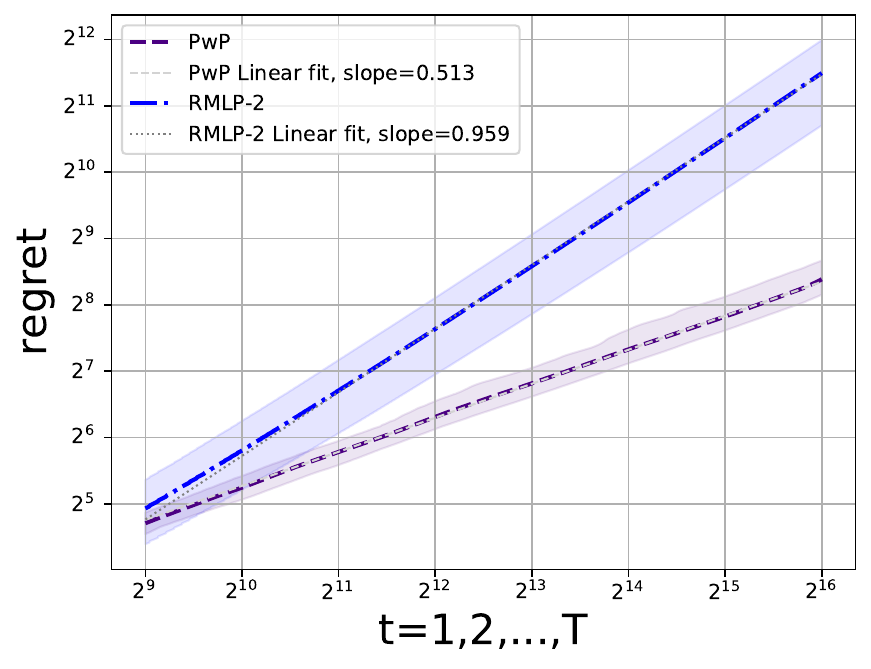}
		\caption{Adversarial $\{x_t\}$'s}\label{fig:adversarial_plot}
	\end{subfigure}
	\caption{\small The regret of PwP algorithm and a modified RMLP-2 algorithm on simulation data (generated according to \Cref{equ:model}), plotted in log-log scales to indicate the regret dependence on $T$. \Cref{fig:stochastic_plot} and \Cref{fig:adversarial_plot} are for stochastic and adversarial $\{x_t\}$ sequences respectively. We also plot linear fits for those regret curves, where a slope-$\alpha$ line indicates an $O(T^\alpha)$ regret. The error bands are drawn with 0.95 coverage using Wald's test. From the figures, we know that PwP performs closely to its $O(\sqrt{T\log T})$ regret regardless of the types of input context sequences, whereas RMLP-2 fails in the attack of adversarial input.} \label{fig:plot_well_assumed}
\end{figure*}

Here we conduct numerical experiments to validate the low-regret performance of our algorithm PwP. Since we are the first to study this heteroscadestic valuation model, we do not have a baseline algorithm working for exactly the same problem. However, we can modify the RMLP-2 algorithm in \citet{javanmard2019dynamic} by only replacing their max-likelihood estimator (MLE) for $\theta^*$ with a new MLE for both $\theta^*$ and $\eta^*$. This modified RMLP-2 algorithm does not have a regret guarantee in our setting, but it may still serve as a baseline to compare with.

We test PwP and the modified RMLP-2 on the demand model assumed in \Cref{equ:model} with both stochastic and adversarial $\{x_t\}$ sequences, respectively. Basically, we assume $T=2^{16}$ $d=2$, $N_t\sim\cN(0, \sigma^2)$ with $\sigma=0.5$, and we repeatedly run each algorithm for 20 times in each experiment setting. In order to show the regret dependence w.r.t. $T$, we plot all cumulative regret curves in log-log plots, where an $\alpha$ slope indicates an $O(T^\alpha)$ dependence.

\paragraph{Stochastic $\{x_t\}$.}
We implement and test PwP and RMLP-2 on stochastic $\{x_t\}$'s, where $x_t$ are iid sampled from $\cN(\mu_x, \Sigma_x)$ (for $\mu_x = [10, 10, \ldots, 10]^\top$ and some randomly sampled $\Sigma_x$) and then normalized s.t. $\|x_t\|_2\leq 1$. The numerical results are shown in \Cref{fig:stochastic_plot}. Numerical results show that both algorithms achieve $\sim O(T^{0.56})$ regrets, which is close to the theoretic regret rate at $O(\sqrt{T\log T})$.

\paragraph{Adversarial $\{x_t\}$.}
Here we design an adversarial $\{x_t\}$ sequence to attack both algorithms. Since RMLP-2 divides the whole time horizon $T$ into epochs with length $k=1,2,3,\ldots$ sequentially and then does pure exploration at the beginning of each epoch, we may directly attack those pure-exploration rounds in the following way: (1) In each pure-exploration round (i.e. when $t=1,3,6,\ldots, \frac{k(k+1)}2, \ldots$), let the context be $x_t=[1,0]^{\top}$; (2) In any other round, let the context be $x_t=[0,1]^{\top}$. In this way, the RMLP-2 algorithm will never learn $\theta^*[2]$ and $\eta^*[2]$ since the inputs of pure-exploration rounds do not contain this information. Under this oblivious adversarial context sequence, we implement PwP and RMLP-2 and compare their performance. The results are shown in \Cref{fig:adversarial_plot}, indicating that PwP can still guarantee $O(T^{0.513})$ regret (close to $O(\sqrt{T\log T})$) while RMLP-2 runs into a linear regret.

As a high-level interpretation, the performance difference is because PwP adopts a "distributed" exploration at every time $t$ while RMLP-2 makes it more "concentrated". Although both PwP and RMLP-2 take the same amount of exploration that optimally balance the reward loss and the information gain (and that is why they both perform well in stochastic inputs), randomly distributed exploration would save the algorithm from being "attacked" by oblivious adversary. In fact, this phenomenon is analog to $\epsilon$-Greedy versus Exploration-first algorithms in multi-armed bandits. We will discuss more in \Cref{appendix:more_discussion}. 

So far, we have presented the numerical results of running PwP and a modified RMLP-2 on the well-assumed demand model as \Cref{equ:model} (or \Cref{equ:valuation_model} equivalently). Besides of that, we also conduct experiments on a model-misspecification setting to show the robustness, where the true demand (or valuation) distribution is not the same as \Cref{equ:model} or \Cref{equ:valuation_model}. The numerical results are presented in \Cref{appendix:more_experiments}.

\section{Discussion}
\label{sec:discussion}
Here we discuss the motivation and the limitation of making \Cref{assumption:bound}. We leave the majority of discussion to \Cref{appendix:more_discussion}.

\textbf{Necessity of lower-bounding $x_t^\top\eta^*$ from 0.}
As we state in \Cref{assumption:bound}, the price elasticity coefficient $x_t^\top\eta^*$ is lower bounded by a constant $C_{\beta}>0$. On the one hand, this is necessary since we cannot have an upper bound on the optimal price without this assumption. On the other hand, according to \Cref{equ:reward_function}, we know that $r(u,\beta, p) = r(u, 1, \beta\cdot p)\cdot\frac1{\beta}$, which indicates that the reward is rescaled by $\frac1{\beta}$. As a result, the regret should be proportional to $\frac1{C_{\beta}}$. Although a larger (i.e. closer to $0$) elasticity would lead to a more \emph{smooth} demand curve, this actually reduce the information we could gather from customers' feedback and slow down the learning process. We look forward to future researches getting rid of this assumption and achieve more adaptive regret rates. 


\section{Conclusion}
\label{sec:conclusion}
In summary, our work focuses on the problem of contextual pricing with highly differentiated products. We propose a contextual elasticity model that unifies the ``linear demand'' and ``linear valuation'' camps and captures the price effect and heteroscedasticity. To solve this problem, we develop an algorithm PwP, which utilizes Online Newton Step (ONS) on a surrogate loss function and proposes perturbed prices for exploration. Our analysis show that it guarantees a $O(\sqrt{dT\log T})$ regret even for adversarial context sequences. We also provide a matching $\Omega(\sqrt{dT})$ regret lower bound to show its optimality (up to $\log T$ factors). Besides, our numerical experiments also validate the regret bounds of PwP and its advantage over existing method. We hope this work would shed lights on the research of contextual pricing as well as online decision-making problems.

\bibliography{ref_log}

\begin{thebibliography}{}

\bibitem[Amin et~al., 2014]{amin2014repeated}
Amin, K., Rostamizadeh, A., and Syed, U. (2014).
\newblock Repeated contextual auctions with strategic buyers.
\newblock In {\em Advances in Neural Information Processing Systems (NIPS-14)},
  pages 622--630.

\bibitem[Anava and Mannor, 2016]{anava2016heteroscedastic}
Anava, O. and Mannor, S. (2016).
\newblock Heteroscedastic sequences: beyond gaussianity.
\newblock In {\em International Conference on Machine Learning}, pages
  755--763. PMLR.

\bibitem[Anderson et~al., 1997]{anderson1997price}
Anderson, P.~L., McLellan, R.~D., Overton, J.~P., and Wolfram, G.~L. (1997).
\newblock Price elasticity of demand.
\newblock {\em McKinac Center for Public Policy. Accessed October}, 13(2).

\bibitem[Baby et~al., 2022]{baby2022non}
Baby, D., Xu, J., and Wang, Y.-X. (2022).
\newblock Non-stationary contextual pricing with safety constraints.
\newblock {\em Transactions on Machine Learning Research}.

\bibitem[Ban and Keskin, 2021]{ban2021personalized}
Ban, G.-Y. and Keskin, N.~B. (2021).
\newblock Personalized dynamic pricing with machine learning: High-dimensional
  features and heterogeneous elasticity.
\newblock {\em Management Science}, 67(9):5549--5568.

\bibitem[Breusch and Pagan, 1979]{breusch1979simple}
Breusch, T.~S. and Pagan, A.~R. (1979).
\newblock A simple test for heteroscedasticity and random coefficient
  variation.
\newblock {\em Econometrica: Journal of the econometric society}, pages
  1287--1294.

\bibitem[Broder and Rusmevichientong, 2012]{broder2012dynamic}
Broder, J. and Rusmevichientong, P. (2012).
\newblock Dynamic pricing under a general parametric choice model.
\newblock {\em Operations Research}, 60(4):965--980.

\bibitem[Bu et~al., 2022]{bu2022context}
Bu, J., Simchi-Levi, D., and Wang, C. (2022).
\newblock Context-based dynamic pricing with partially linear demand model.
\newblock In {\em Advances in Neural Information Processing Systems}.

\bibitem[Chaudhuri et~al., 2017]{chaudhuri2017active}
Chaudhuri, K., Jain, P., and Natarajan, N. (2017).
\newblock Active heteroscedastic regression.
\newblock In {\em International Conference on Machine Learning}, pages
  694--702. PMLR.

\bibitem[Chen et~al., 2023]{chen2023utility}
Chen, X., Simchi-Levi, D., and Wang, Y. (2023).
\newblock Utility fairness in contextual dynamic pricing with demand learning.
\newblock {\em arXiv preprint arXiv:2311.16528}.

\bibitem[Cohen et~al., 2022]{cohen2022price}
Cohen, M.~C., Elmachtoub, A.~N., and Lei, X. (2022).
\newblock Price discrimination with fairness constraints.
\newblock {\em Management Science}.

\bibitem[Cohen et~al., 2020]{cohen2020feature_journal}
Cohen, M.~C., Lobel, I., and Paes~Leme, R. (2020).
\newblock Feature-based dynamic pricing.
\newblock {\em Management Science}, 66(11):4921--4943.

\bibitem[Cohen et~al., 2021]{cohen2021dynamic}
Cohen, M.~C., Miao, S., and Wang, Y. (2021).
\newblock Dynamic pricing with fairness constraints.
\newblock {\em Available at SSRN 3930622}.

\bibitem[Cunia, 1964]{cunia1964weighted}
Cunia, T. (1964).
\newblock Weighted least squares method and construction of volume tables.
\newblock {\em Forest Science}, 10(2):180--191.

\bibitem[Fan et~al., 2021]{fan2021policy}
Fan, J., Guo, Y., and Yu, M. (2021).
\newblock Policy optimization using semiparametric models for dynamic pricing.
\newblock {\em arXiv preprint arXiv:2109.06368}.

\bibitem[Gale, 1955]{gale1955law}
Gale, D. (1955).
\newblock The law of supply and demand.
\newblock {\em Mathematica scandinavica}, pages 155--169.

\bibitem[Golrezaei et~al., 2019]{golrezaei2019incentive}
Golrezaei, N., Jaillet, P., and Liang, J. C.~N. (2019).
\newblock Incentive-aware contextual pricing with non-parametric market noise.
\newblock {\em arXiv preprint arXiv:1911.03508}.

\bibitem[Goyal and Perivier, 2021]{goyal2021dynamic}
Goyal, V. and Perivier, N. (2021).
\newblock Dynamic pricing and assortment under a contextual mnl demand.
\newblock {\em arXiv preprint arXiv:2110.10018}.

\bibitem[Hazan, 2016]{hazan2019introduction}
Hazan, E. (2016).
\newblock Introduction to online convex optimization.
\newblock {\em Foundations and Trends in Optimization}, 2(3-4):157--325.

\bibitem[Hildenbrand, 1983]{hildenbrand1983law}
Hildenbrand, W. (1983).
\newblock On the" law of demand".
\newblock {\em Econometrica: Journal of the Econometric Society}, pages
  997--1019.

\bibitem[Hutter et~al., 2005]{hutter2005adaptive}
Hutter, M., Poland, J., et~al. (2005).
\newblock Adaptive online prediction by following the perturbed leader.

\bibitem[Javanmard and Nazerzadeh, 2019]{javanmard2019dynamic}
Javanmard, A. and Nazerzadeh, H. (2019).
\newblock Dynamic pricing in high-dimensions.
\newblock {\em The Journal of Machine Learning Research}, 20(1):315--363.

\bibitem[Kaufman, 2013]{kaufman2013heteroskedasticity}
Kaufman, R.~L. (2013).
\newblock {\em Heteroskedasticity in regression: Detection and correction}.
\newblock Sage Publications.

\bibitem[Kleinberg and Leighton, 2003]{kleinberg2003value}
Kleinberg, R. and Leighton, T. (2003).
\newblock The value of knowing a demand curve: Bounds on regret for online
  posted-price auctions.
\newblock In {\em IEEE Symposium on Foundations of Computer Science (FOCS-03)},
  pages 594--605. IEEE.

\bibitem[Krishnamurthy et~al., 2021]{krishnamurthy2020contextual}
Krishnamurthy, A., Lykouris, T., Podimata, C., and Schapire, R. (2021).
\newblock Contextual search in the presence of irrational agents.
\newblock In {\em Proceedings of the 53rd Annual ACM SIGACT Symposium on Theory
  of Computing (STOC-21)}, pages 910--918.

\bibitem[Lai and Robbins, 1985]{lai1985asymptotically}
Lai, T.~L. and Robbins, H. (1985).
\newblock Asymptotically efficient adaptive allocation rules.
\newblock {\em Advances in applied mathematics}, 6(1):4--22.

\bibitem[Leme et~al., 2021]{leme2021learning}
Leme, R.~P., Sivan, B., Teng, Y., and Worah, P. (2021).
\newblock Learning to price against a moving target.
\newblock In {\em International Conference on Machine Learning}, pages
  6223--6232. PMLR.

\bibitem[Luo et~al., 2021]{luo2021distribution}
Luo, Y., Sun, W.~W., et~al. (2021).
\newblock Distribution-free contextual dynamic pricing.
\newblock {\em arXiv preprint arXiv:2109.07340}.

\bibitem[Luo et~al., 2022]{luo2022contextual}
Luo, Y., Sun, W.~W., and Liu, Y. (2022).
\newblock Contextual dynamic pricing with unknown noise: Explore-then-ucb
  strategy and improved regrets.
\newblock In {\em Advances in Neural Information Processing Systems}.

\bibitem[Marshall, 2009]{marshall2009principles}
Marshall, A. (2009).
\newblock {\em Principles of economics: unabridged eighth edition}.
\newblock Cosimo, Inc.

\bibitem[Miao et~al., 2019]{miao2019context}
Miao, S., Chen, X., Chao, X., Liu, J., and Zhang, Y. (2019).
\newblock Context-based dynamic pricing with online clustering.
\newblock {\em arXiv preprint arXiv:1902.06199}.

\bibitem[Nambiar et~al., 2019]{nambiar2019dynamic}
Nambiar, M., Simchi-Levi, D., and Wang, H. (2019).
\newblock Dynamic learning and pricing with model misspecification.
\newblock {\em Management Science}, 65(11):4980--5000.

\bibitem[Parkin et~al., 2002]{parkin2002economics}
Parkin, M., Powell, M., and Matthews, K. (2002).
\newblock {\em Economics}.
\newblock Addison-Wesley, Harlow.

\bibitem[Qiang and Bayati, 2016]{qiang2016dynamic}
Qiang, S. and Bayati, M. (2016).
\newblock Dynamic pricing with demand covariates.
\newblock {\em arXiv preprint arXiv:1604.07463}.

\bibitem[Shah et~al., 2019]{shah2019semi}
Shah, V., Johari, R., and Blanchet, J. (2019).
\newblock Semi-parametric dynamic contextual pricing.
\newblock {\em Advances in Neural Information Processing Systems}, 32.

\bibitem[Wang et~al., 2021a]{wang2021dynamic}
Wang, H., Talluri, K., and Li, X. (2021a).
\newblock On dynamic pricing with covariates.
\newblock {\em arXiv preprint arXiv:2112.13254}.

\bibitem[Wang et~al., 2021b]{wang2021multimodal}
Wang, Y., Chen, B., and Simchi-Levi, D. (2021b).
\newblock Multimodal dynamic pricing.
\newblock {\em Management Science}.

\bibitem[White, 1980]{white1980heteroskedasticity}
White, H. (1980).
\newblock A heteroskedasticity-consistent covariance matrix estimator and a
  direct test for heteroskedasticity.
\newblock {\em Econometrica: journal of the Econometric Society}, pages
  817--838.

\bibitem[Xu et~al., 2023]{xu2023doubly}
Xu, J., Qiao, D., and Wang, Y.-X. (2023).
\newblock Doubly fair dynamic pricing.
\newblock In {\em International Conference on Artificial Intelligence and
  Statistics}, pages 9941--9975. PMLR.

\bibitem[Xu and Wang, 2021]{xu2021logarithmic}
Xu, J. and Wang, Y.-X. (2021).
\newblock Logarithmic regret in feature-based dynamic pricing.
\newblock {\em Advances in Neural Information Processing Systems}, 34.

\bibitem[Xu and Wang, 2022]{xu2022towards}
Xu, J. and Wang, Y.-X. (2022).
\newblock Towards agnostic feature-based dynamic pricing: Linear policies vs
  linear valuation with unknown noise.
\newblock {\em International Conference on Artificial Intelligence and
  Statistics (AISTATS)}.

\bibitem[Zhou, 2015]{zhou2015survey}
Zhou, L. (2015).
\newblock A survey on contextual multi-armed bandits.
\newblock {\em arXiv preprint arXiv:1508.03326}.

\end{thebibliography}

\newpage
\appendix
\onecolumn
\section{Definition and Proof Details}
\label{appendix: proof}
Here we show the proof details of the lemmas we have stated in \Cref{sec:regret_analysis}. Before that, let us clarify some terminologies we mentioned in the main paper.

\subsection{Definitions}
\label{appendix:definition}
Firstly, we rigorously define the concept of \emph{price elasticity} occurring in \Cref{sec:introduction}.

\begin{definition}[Price Elasticity \citep{parkin2002economics}]
	Suppose $D(p)$ is a demand function of price $p$. Then the price elasticity $E_d$ of demand is defined as
	\begin{equation}
	\label{equ:def_elasticity}
	E_D := \frac{\Delta D(p)/D(p)}{\Delta p / p} = \frac{\partial D(p)}{\partial p}\cdot \frac{p}{D(p)}.
	\end{equation} 
	\label{def:elasticity}
\end{definition}

With this definition, along with our generalized linear demand model given in \Cref{equ:model}, the price elasticity for the expected demand $S(x_t^\top\eta^*\cdot p_t -x_t^\top\theta^*)$ is
\begin{equation}
\label{equ:derive_elasticity_of_our_demand}
\begin{aligned}
E_D =& \frac{\partial S(x_t^\top\eta^*\cdot p_t -x_t^\top\theta^*)}{\partial p_t} \cdot \frac{p_t}{S(x_t^\top\eta^*\cdot p_t -x_t^\top\theta^*)}\\
=&x_t^\top\eta^*\cdot\frac{s(x_t^\top\eta^*\cdot p_t -x_t^\top\theta^*)}{S(x_t^\top\eta^*\cdot p_t -x_t^\top\theta^*)}\cdot p_t.
\end{aligned}
\end{equation}

Here $s(\cdot)=S'(\cdot)$. Therefore, despite the effect of the link function and the price $p_t$, the price elasticity is proportional to the price coefficient $x_t^\top\eta^*$. This is why we call $x_t^\top\eta^*$ (or $\alpha$ in the general model $D(p)=\lambda(\alpha\cdot p  + x_t^T \beta)$) the \emph{elasticity coefficient} or \emph{coefficient of elasticity} in \Cref{sec:introduction}.

\subsection{Proof of \Cref{lemma:smooth_regret}}
\label{subsec:proof_lemma_smooth_regret}

\begin{proof}
In order to prove \Cref{lemma:smooth_regret}, we show the following lemma that indicates the Lipschitzness of $J(u,\beta)$:
\begin{lemma}[Lipschitz of optimal price]
	\label{lemma:lipschitz_J}
	There exists a constant $C_J>0$ such that
	\begin{equation}\label{equ:lipschitz_j}
	|J(u_1, \beta_1) - J(u_2, \beta_2)|\leq C_J\cdot (|u_1-u_2| + |\beta_1-\beta_2|).
	\end{equation}
\end{lemma}
With this lemma, we get the second inequality of \Cref{equ:regret_smooth}. We will prove this lemma later in this subsection. Now, we focus on the first inequality.	
	Notice that
	\begin{equation}
	\label{equ:quadratic_regret_detail}
	\begin{aligned}
	Reg_t(p_t)=&r(x_t^\top\theta^*, x_t^\top\eta^*, p_t^*)-r(x_t^\top\theta^*, x_t^\top\eta^*, p_t)\\
	\leq&-\frac{\partial r(u, \beta, p)}{\partial p}|_{u=x_t^\top\theta^*, \beta=x_t^\top\eta^*, p=p_t^*}(p_t^*-p_t)\\
	&- \frac12\cdot\inf_{p\in[c_1, c_2], \beta\in[C_{\beta}, 1], u\in[0,1]}\frac{\partial^2 r(u, \beta, p)}{\partial p^2}|_{u=x_t^\top\theta^*, \beta=x_t^\top\eta^*, p=p_t^*}(p_t^*-p_t)^2\\
	=&0 + \frac12\cdot\sup_{p\in[c_1, c_2], \beta\in[C_{\beta}, 1], u\in[0,1]}\{|2s(\beta\cdot p - u)\cdot\beta + p\cdot s'(\beta\cdot p - u)\cdot\beta^2|\}(p_t^*-p_t)^2.\\
	\end{aligned}
	\end{equation}
	Here the first line is by the definition of $Reg_t(p_t)$, the second line is by smoothness, the third line is by the optimality of $p_t^*$, and the last line is by calculus. Since $|2s(\beta\cdot p - u)\cdot\beta + p\cdot s'(\beta\cdot p - u)\cdot\beta^2|$ is continuous on $p\in[c_1, c_2], \beta\in[C_{\beta}, 1], u\in[0,1]$, we denote this maximum as $2C_r$, which proves the first inequality of \Cref{equ:regret_smooth}.
\end{proof}

Now we show the proof of \Cref{lemma:lipschitz_J}.
\begin{proof}[Proof of \Cref{lemma:lipschitz_J}]
	Since $J(u, \beta) = \frac{u + \varphi^{-1}(u)}{\beta}$ where $\varphi(w) = -\frac{S(w)}{s(w)}-w$. Notice that
	\begin{equation}
	\label{equ:phi_inv_derivative}
	\varphi'(w) = -\frac{d \frac{S(w)}{s(w)}}{d w} - 1 = \frac{d^2 \log(S(w))}{dw^2}\cdot\frac{S(w)^2}{s(w)^2} - 1 < -1
	\end{equation}
	since $S(w)$ is log-concave (as is assumed in \Cref{assumption:log-concave}). Given \Cref{equ:phi_inv_derivative}, we know that $\frac{d\varphi^{-1}(u)}{d(u)}=\frac1{\frac{d\varphi(w)}{dw}|_{w=\varphi^{-1}(u)}}\in(-1,0)$. Therefore, we have:
	\begin{equation}
	\label{equ:j_partial_derivatives}
	\begin{aligned}
	\frac{\partial J(u,\beta)}{\partial u} =& \frac{1 + \frac{d\varphi^{-1}(u)}{d u}}{\beta}\in(0, \frac1{C_{\beta}})\\
	\frac{\partial J(u, \beta)}{\partial \beta} =&\frac{\partial \frac{J(u, 1)}{\beta}}{\partial \beta}= -\frac{J(u, 1)}{\beta^2}\in[-\frac{c_2}{C_{\beta}}, -c_1].
	\end{aligned}
	\end{equation}
	Therefore, we know that $J(u, \beta)$ is Lipschitz with respect to $u$ and $\beta$ respectively. Take $C_J=\max\{\frac1{C_{\beta}}, \frac{c_2}{C_{\beta}}\}$ and we get \Cref{equ:lipschitz_j}.
\end{proof}

\subsection{Proof of \Cref{lemma:surrogate_expected_regret}}
\label{subsec:proof_lemma_surrogate_expected_regret}
\begin{proof}
We firstly show the convexity (and exp-concavity) of $\ell_t(\theta, \eta)$ by the following lemma.
\begin{lemma}[exp-concavity]
	\label{lemma:exp_concavity}
	$\ell_t(\theta, \eta)$ is convex and $C_e$-exp-concave with respect to $[\theta; \eta]$, where $C_e>0$ is a constant dependent on $F$ and $C_\beta$. Equivalently, $\nabla^2 \ell_t(\theta, \eta)\succeq C_e \cdot \nabla\ell_t(\theta, \eta)\nabla\ell_t(\theta, \eta)^\top$. Also, we have $\nabla^2\ell_t(\theta, \eta)\succeq C_l\cdot\begin{bmatrix}
	x_tx_t^\top & -p_t\cdot x_tx_t^\top\\
	-p_t\cdot x_tx_t^\top & v_t^2\cdot x_tx_t^\top
	\end{bmatrix}$ for some constant $C_l>0$.
\end{lemma}
The proof of \Cref{lemma:exp_concavity} is mainly straightforward calculus, and we defer the proof to the end of this subsection. According to \Cref{lemma:exp_concavity}, we have $\nabla^2\ell_t(\theta, \eta)\succeq C_l\cdot\begin{bmatrix}
x_tx_t^\top & -p_t\cdot x_tx_t^\top\\
-p_t\cdot x_tx_t^\top & p_t^2\cdot x_tx_t^\top
\end{bmatrix}$. Therefore, we know that
\begin{equation}
\label{equ:proof_convexity}
\begin{aligned}
\ell_t(\theta, \eta)\geq&\ell_t(\theta^*, \eta^*)+\nabla\ell_t(\theta^*, \eta^*)^\top \left[\begin{array}{cc}
\theta-\theta^*\\
\eta-\eta^*
\end{array}
\right] + [(\theta-\theta^*)^\top, (\eta-\eta^*)^\top]C_l
\left[\begin{array}{cc}
x_tx_t^\top & -p_tx_tx_t^\top\\
-p_tx_tx_t^\top & p_t^2x_tx_t^\top
\end{array}
\right]
\left[\begin{array}{cc}
\theta-\theta^*\\
\eta-\eta^*
\end{array}
\right]
\end{aligned}
\end{equation}
According to the property of likelihood, we have $\E[\nabla\ell_t(\theta^*, \eta^*)|\theta_t, \eta_t] = 0$ for any $\theta_t$ and $\eta_t$. Combining this with \Cref{equ:proof_convexity}, we get
\begin{equation}
\label{equ:proof_quadratic_term}
\begin{aligned}
\E[\ell_t(\theta, \eta)-\ell_t(\theta^*, \eta^*)|\theta_t, \eta_t]&\geq C_l [(\theta-\theta^*)^\top, (\eta-\eta^*)^\top]
\E\left[\begin{array}{cc}
x_tx_t^\top & -p_tx_tx_t^\top\\
-p_tx_tx_t^\top & p_t^2x_tx_t^\top
\end{array}
|\theta_t, \eta_t\right]
\left[\begin{array}{cc}
\theta-\theta^*\\
\eta-\eta^*
\end{array}
\right]
\end{aligned}
\end{equation}
Recall that $\hat p_t = J(x_t^\top\theta_t, x_t^\top\eta_t)$ and that $p_t = \hat p_t + \Delta_t$. Therefore, we know that the conditional expectations $\E[p_t|\theta_t, \eta_t] = \hat p_t$ and $\E[p_t^2|\theta_t, \eta_t] = \hat p_t^2 + \Delta^2$. Given this, we have
\begin{equation}
\label{equ:expected_hassion}
\begin{aligned}
\E&\left[\begin{array}{cc}
x_tx_t^\top & -p_tx_tx_t^\top\\
-p_tx_tx_t^\top & p_t^2x_tx_t^\top
\end{array}
|\theta_t, \eta_t\right]\\
=\quad &\left[\begin{array}{cc}
x_tx_t^\top & -\hat p_tx_tx_t^\top\\
-\hat p_tx_tx_t^\top & (\hat p_t^2+\Delta^2)x_tx_t^\top
\end{array}
\right]\\
=\quad &\left[\begin{array}{c}
x_t\\
-\hat p_tx_t\\
\end{array}
\right]\left[x_t^\top, -\hat p_tx_t^\top\right] +\left[
\begin{array}{cc}
0 & 0 \\
0 & \Delta^2x_tx_t^\top
\end{array}
\right]\\
=\quad & \left[
\begin{array}{c}
\frac1{\sqrt{1+\frac{\Delta^2}2}}\cdot x_t\\
-\sqrt{1+\frac{\Delta^2}2}\hat p_t\cdot x_t
\end{array}
\right]\left[\frac1{\sqrt{1+\frac{\Delta^2}2}}\cdot x_t^\top, -\sqrt{1+\frac{\Delta^2}2}\hat p_t\cdot x_t^\top
\right] + \left[
\begin{array}{cc}
(1-\frac1{1+\frac{\Delta^2}2})x_tx_t^\top & 0\\
0 & \frac{\Delta^2}2 x_tx_t^\top
\end{array}
\right]
\end{aligned}
\end{equation}
Since $\Delta = \min\left\{\left(\frac{d\log T}{T}\right)^{\frac14}, \frac{J(0,1)}{10}, \frac1{10}\right\}$, we have $1-\frac1{1+\frac{\Delta^2}2}=\frac{\frac{\Delta^2}2}{1+\frac{\Delta^2}2}\geq\frac{\Delta^2}{10}$. As a result, we have
\begin{equation}
\label{equ:expected_hassion_lower_bound}
\begin{aligned}
\E&\left[\begin{array}{cc}
x_tx_t^\top & -p_tx_tx_t^\top\\
-p_tx_tx_t^\top & p_t^2x_tx_t^\top
\end{array}
|\theta_t, \eta_t\right]\\
\geq\quad&\frac{\Delta^2}{10}\cdot\left[
\begin{array}{cc}
x_tx_t^\top & 0 \\
0 & x_tx_t^\top
\end{array}
\right]
\end{aligned}
\end{equation}
This proves the lemma.
\end{proof}
Finally, we show the proof of \Cref{lemma:exp_concavity}.
\begin{proof}[Proof of \Cref{lemma:exp_concavity}]
	Recall that $\ell_t(\theta, \eta) = -\ind_t\cdot\log(S(x_t^\top(p_t\eta-\theta))) - (1-\ind_t)\cdot\log (1-S(x_t^\top(p_t\eta-\theta)))$. We first calculate the gradient and Hessian of $\ell_t(\theta, \eta)$ with respect to $[\theta; \eta]$. For notation simplicity, denote $w_t := x_t^\top(p_t\eta-\theta)$.
	\begin{equation}
	\label{equ:gradient_ell_t}
	\begin{aligned}
	\grad\ell_t = -\left(\ind_t\cdot\frac{s(w_t)}{S(w_t)} - (1-\ind_t)\cdot\frac{s(w_t)}{1-S(w_t)}\right)\cdot\left[
	\begin{array}{c}
	-x_t\\
	p_t\cdot x_t
	\end{array}
	\right]
	\end{aligned}
	\end{equation}
	\begin{equation}
	\label{equ:hessian_ell_t}
	\begin{aligned}
	\grad^2\ell_t =& -\left(\ind_t\cdot\frac{s'(w_t)S(w_t) - s(w_t)^2}{S(w_t)^2} + (1-\ind_t)\cdot\frac{-s'(w_t)(1-S(w_t))-s(w_t)^2}{(1-S(w_t))^2}\right)\cdot\left[
	\begin{array}{c}
	-x_t\\
	p_t\cdot x_t
	\end{array}
	\right]\left[-x_t^\top, \ p_t\cdot x_t^\top\right]\\
	=& -\left(\ind_t\cdot\frac{s'(w_t)S(w_t) - s(w_t)^2}{S(w_t)^2} + (1-\ind_t)\cdot\frac{-s'(w_t)(1-S(w_t))-s(w_t)^2}{(1-S(w_t))^2}\right)\cdot\left[
	\begin{array}{cc}
	x_tx_t^\top & -p_t\cdot x_tx_t^\top\\
	-p_t\cdot x_tx_t^\top & p_t^2\cdot x_tx_t^\top
	\end{array}\right]
	\end{aligned}
	\end{equation}
	According to \Cref{assumption:log-concave}, we know that $S(w)$ and $(1-S(w))$ are strictly log-concave, which indicates
	\begin{equation}
	\label{equ:log_concave_expanded}
	\begin{aligned}
	\frac{d^2 \log(1-S(w))}{dw^2}=&\frac{-s'(w)(1-S(w))-s(w)^2}{(1-S(w))^2}<0\\
	\frac{d^2 \log(S(w))}{dw^2}=&\frac{s'(w)S(w)-s(w)^2}{S(w)^2}<0, \forall w\in\R.
	\end{aligned}
	\end{equation}
	Since $w_t = p_t\cdot x_t^\top\eta-x_t^\top\theta$ where $p_t\in[c_1, c_2]$, we know that $w_t\in[-1, c_2]$. Since $\frac{d^2 \log(S(w))}{dw^2}$ and $\frac{d^2 \log(1-S(w))}{dw^2}$ are continuous on $[-1, c_2]$, we know that
	\begin{equation}
	\label{equ:exact_lower_constant_bound_on_hessian}
	\begin{aligned}
	&\ind_t\cdot\frac{s'(w_t)S(w_t) - s(w_t)^2}{S(w_t)^2} + (1-\ind_t)\cdot\frac{-s'(w_t)(1-S(w_t))-s(w_t)^2}{(1-S(w_t))^2} \\
	\leq& \sup_{w\in[-1, c_2]}\max\left\{\frac{s'(w_t)S(w_t) - s(w_t)^2}{S(w_t)^2}, \frac{-s'(w_t)(1-S(w_t))-s(w_t)^2}{(1-S(w_t))^2} \right\}<0.
	\end{aligned}
	\end{equation}
	Denote $C_l=-\sup_{w\in[-1, c_2]}\max\left\{\frac{s'(w_t)S(w_t) - s(w_t)^2}{S(w_t)^2}, \frac{-s'(w_t)(1-S(w_t))-s(w_t)^2}{(1-S(w_t))^2} \right\}>0$, and we know that
	\begin{equation}
	\grad^2 \ell_t(\theta, \eta)\succeq C_l\cdot\left[
	\begin{array}{cc}
	x_tx_t^\top & -p_t\cdot x_tx_t^\top\\
	-p_t\cdot x_tx_t^\top & p_t^2\cdot x_tx_t^\top
	\end{array}
	\right].
	\end{equation}
	Similarly, we know that $\frac{s(w)}{S(w)}$ and $\frac{-s(w)}{1-S(w)}$ are continuous on $[-1, c_2]$. Therefore, we may denote $C_G=\sup_{w\in[-1, c_2]}\max\left\{|\frac{s(w)}{S(w)}|, |\frac{-s(w)}{1-S(w)}| \right\}>0$ and get
	\begin{equation}
	\grad\ell_t(\theta, \eta)\grad\ell_t(\theta, \eta)^\top\preceq C_G^2\cdot\left[
	\begin{array}{c}
	-x_t\\
	p_t\cdot x_t
	\end{array}
	\right]\left[-x_t^\top, \ p_t\cdot x_t^\top\right].
	\end{equation} 
	Given all these above, we have
	\begin{equation}
	\label{equ:log_concave_proved}
	\begin{aligned}
	\grad^2 \ell_t(\theta, \eta)&\succeq C_l\cdot\left[
	\begin{array}{cc}
	x_tx_t^\top & -p_t\cdot x_tx_t^\top\\
	-p_t\cdot x_tx_t^\top & p_t^2\cdot x_tx_t^\top
	\end{array}
	\right]\\
	&= \frac{C_l}{C_G^2}\cdot C_G^2\cdot\left[
	\begin{array}{c}
	-x_t\\
	p_t\cdot x_t
	\end{array}
	\right]\left[-x_t^\top, \ p_t\cdot x_t^\top\right]\\
	&\succeq \frac{C_l}{C_G^2}\cdot\grad\ell_t(\theta, \eta)\grad\ell_t(\theta, \eta)^\top.
	\end{aligned}
	\end{equation}
	Denote $C_e:=\frac{C_l}{C_G^2}$ and we prove the lemma.
\end{proof}

\subsection{Proof of \Cref{thm: main_regret}}
\label{subsec:proof_main_regret}
\begin{proof}
With all lemmas above, we have
	\begin{equation}
	\label{equ:expected_regret_bound}
	\begin{aligned}
	\E[Regret]=&\E[\sum_{t=1}^T\E[Reg_t(p_t)|\theta_t, \eta_t]]\\
	\leq&\E[\sum_{t=1}^TC_r\cdot2\cdot C_J\cdot\E[(x_t^\top(\theta_t-\theta^*))^2+(x_T^\top(\eta_t-\eta^*))^2|\theta_t, \eta_t] + T\cdot C_r \cdot2\cdot\Delta^2]\\
	\leq&\E[\sum_{t=1}^T2C_r C_J\cdot\frac{10}{C_l\cdot\Delta^2}\cdot\E[\ell_t(\theta_t, \eta_t)-\ell_t(\theta^*, \eta^*)|\theta_t, \eta_t] + 2C_r T \Delta^2]\\
	=&\frac{20C_r C_J}{C_l\Delta^2}\E[\sum_{t=1}^T\ell_t(\theta_t, \eta_t)-\ell_t(\theta^*, \eta^*)] + 2C_r T \Delta^2\\
	=&O(\frac{d\log T}{\Delta^2} + \Delta^2T)\\
	=&O(\sqrt{dT\log T}).
	\end{aligned}
	\end{equation}
	
	Here the first line is by the law of total expectation, the second line is by \Cref{lemma:smooth_regret}, the third line is by \Cref{lemma:surrogate_expected_regret}, the fourth line is by equivalent transformation, the fifth line is by \Cref{lemma:onsp}, and the sixth line is by the fact that $\Delta = \min\left\{\left(\frac{d\log T}{T}\right)^{\frac14}, \frac{J(0,1)}{10}, \frac1{10}\right\}$.
	This holds the theorem.
\end{proof}

\section{More Experiments}
\label{appendix:more_experiments}

\subsection{Model Adaptivity}
\label{subsec:model_adaptivity}
\begin{figure}[t]
	\centering
	\includegraphics[width=0.57\textwidth]{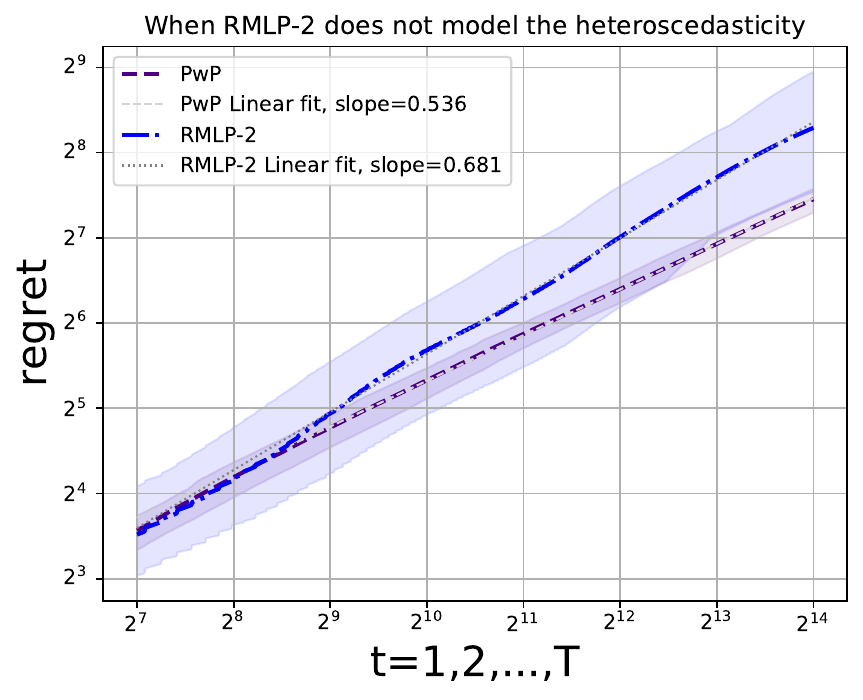}
	\caption{Regrets of PwP versus the original homoscedastic RMLP-2 algorithm. In this log-log diagram, a $O(T^{\alpha})$ regret curve is shown as a straight line with slope $\alpha$. From the figure, we notice that PwP is optimal while RMLP-2 is sub-optimal, indicating the necessity of modeling homoscedasticity to achieve optimal regrets.}
\end{figure}
In this section, we show that it is necessary to model the heteroscedasticity. In specific, we compare PwP with the original RMLP-2 algorithm from \citet{javanmard2019dynamic} that ignores heteroscedasticity in a heteroscedastic environment. We conduct both experiments for $T=2^14$ rounds and repeat them for 10 epochs. The numerical results are displayed in the lower figure, plotted in log-log diagrams. From the figure, we notice that the regret of RMLP-2 is much larger than PwP. Also, the slope of regrets of RMLP-2 is $0.681>>0.5$, indicating that it does not guarantee a $O(\sqrt{T})$ regret. In comparison, PwP still performs well as it achieves a $\sim O(T^{0.536})$ regret. This indicates that the algorithmic adaptivity of PwP to both homoscedastic and heteroscedastic environments is highly non-trivial, and a failure of capturing it would result in a substantial sub-optimality.

\subsection{Model Misspecification}
\label{subsec:model_misspecification}
\begin{figure}[t]
	\centering
	\includegraphics[width= 0.49\textwidth]{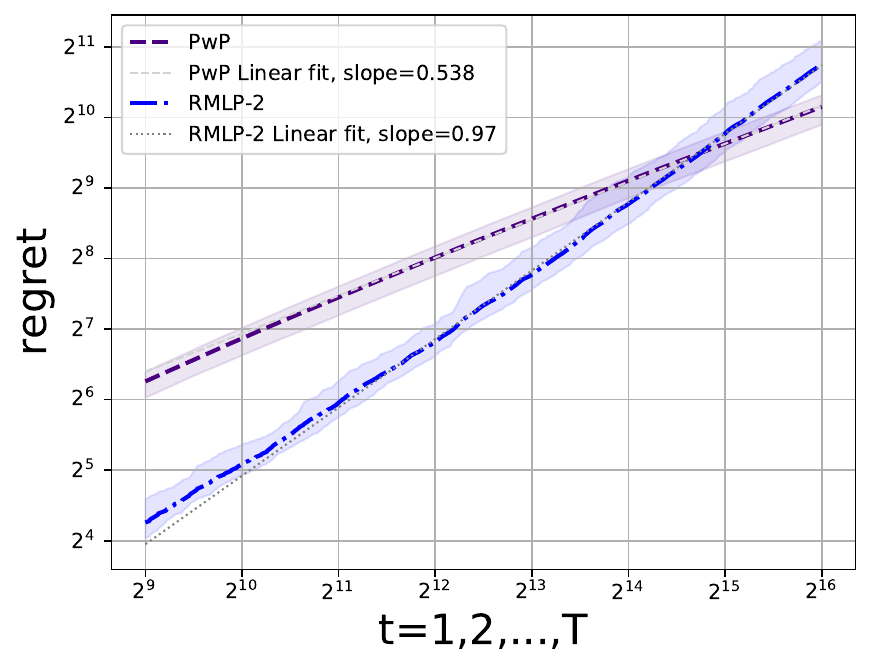}
	\caption{Regrets of misspecified PwP with expanded contexts, in comparison with a baseline RMLP-2 knowing the correct model. The results show that PwP still have a sub-linear regret in a certain period of time with context expansions, indicating that our linear demand model as \Cref{equ:model} can be generalized to a linear valuation model as \Cref{equ:mis_specified_model} in practice.}
	\label{fig:plot_misspecified}
\end{figure}
In \Cref{sec:numerical_experiments}, we compare the cumulative regrets of our PwP algorithm with the (modified) RMLP-2 on the \emph{linear demand} model (as \Cref{equ:model} or equivalently, the linear fractional valuation model as \Cref{equ:valuation_model}). In this section, we consider a model-misspecific setting, where customer's true valuation is given by the following equation
\begin{equation}
\label{equ:mis_specified_model}
y_t = x_t^\top\theta^* + x_t^\top\eta^*\cdot N_t
\end{equation}
and the demand $D_t = \ind_t=\ind[p_t\leq y_t]$. As a result, \Cref{equ:mis_specified_model} captures a \emph{linear valuation} model with heteroscedastic valuation.

Now, we design an experiment to show the generalizability of both our PwP algorithm and our demand model as \Cref{equ:model}. In specific, we run the PwP algorithm that still models a customer's valuation as $\tilde{y}_t = \frac{x_t^\top\tilde{\theta}^* + \tilde{N}_t}{\tilde{x}_t^\top\tilde{\eta}^*}$, where $\tilde{x}_t\in\R^q$ is an \emph{expanded} version of the original context $x_t$ (i.e. $\tilde{x}_t = \pi(x_t)$ for some fixed expanding policy $\pi$) and $\tilde{\theta}^*, \tilde{\eta}^*\in\R^q$ are some fixed parameters\footnote{We may assume $q\geq d$ without loss of generality.}. Therefore, PwP is trying to learn those misspecified $\tilde{\theta}^*$ and $\tilde{\eta}^*$ although there does not exist such an underground truth.

We are curious whether the expansion of context (from $x_t$ to $\tilde{x}_t$) would leverage the hardness of model misspecification. For $x = [x_1, x_2, \ldots, x_d]^\top$, denote $x^n:=[x_1^n, x_2^n, \ldots, x_d^n]^\top$. Then for any context $x\in\R^d$, we specify each context-expanding policy as follows:
\begin{equation}
\label{equ:expanding_policy}
\begin{aligned}
&\pi(x; x_0, \mathbf{a})\\
:=& [x; (x-x_0)^{a_1}; (x-x_0)^{a_2}; \ldots; (x-x_0)^{a_m}]^\top\in\R^{(m+1)d}.
\end{aligned}
\end{equation}
The policy $\pi$ in \Cref{equ:expanding_policy} is a polynomial expansion of $x$ with index list $\mathbf{a}=[a_1, a_2, \ldots, a_m]\in\mathbb{Z}^m$, where $x_0\in\R^d$ is a fixed start point of this expansion.

Now we consider the baseline to compare with. We claim that it is very challenging to solve the contextual pricing problem with customers' valuations being \Cref{equ:mis_specified_model} with theoretic regret guarantees (although the $\Omega(\sqrt{T})$ lower bound given by \citet{javanmard2019dynamic} still holds), and there are no existing algorithms targeting at this problem setting. However, there are still some straightforward algorithms that might approach it: For example, a max-likelihood estimate (MLE) of $\theta^*$ and $\eta^*$. In fact, we may still reuse the framework of RMLP-2 by replacing its MLE oracle according to the distribution given by \Cref{equ:mis_specified_model}. In the following, we will compare the performances of

\begin{enumerate}
	\item PwP algorithm with the misspecified linear demand model as \Cref{equ:model}, with expanded context $\{x_t\}$'s, and
	\item RMLP-2 algorithm on the correct linear valuation model as\Cref{equ:mis_specified_model}, with original context $\{x_t\}$'s.
\end{enumerate}

We implement PwP and RMLP-2 on stochastic $\{x_t\}$ sequences (since RMLP-2 has already failed in the adversarial setting) and get numerical results shown as \Cref{fig:plot_misspecified}. Here we choose $x_0 = [0.5, 0.5]^\top$ and $\mathbf{a} = [0,1]$. For a model-misspecified online-learning algorithm, there generally exists an $O(\epsilon\cdot T)$ term in the regret rate, where $\epsilon$ is a parameter measuring the distance between the global optimal policy and the best \emph{proper} policy (i.e. the best policy in the hypothesis set). However, our numerical results imply that PwP may still achieve a sub-linear regret within a certain time horizon $T$, whereas the baseline RMLP-2 that takes the correct model has a much worse regret. It is worth mentioning that PwP may still run into $\Omega(T)$ regret as $T$ gets sufficiently large, due to model misspecification. These results imply that
\begin{enumerate}
	\item Our linear demand model \Cref{equ:model} can be generalized to a linear valuation model as \Cref{equ:mis_specified_model} in practice.
	\item Our PwP algorithm can still perform well in model-misspecification settings, and even better than a baseline MLE algorithm with a correct model in a certain period of time.
\end{enumerate}

For the first phenomenon that our demand model can be generalized with context expansion tricks, we may understand it as a Taylor expansion (and we take a linear approximation) at $x_0=[0.5, 0.5]^\top$. For the second phenomenon that PwP outperforms RMLP-2, it might be caused by the non-convexity of the log-likelihood function of the valuation model specified in \Cref{equ:mis_specified_model}. As a result, while RMLP-2 is solving a non-convex MLE and getting estimates far from the true parameters, PwP instead works on an online convex optimization problem within a larger space (which probably contain the underground truth) due to context expansions. Unfortunately, we do not have a rigorous analysis of those two phenomenons.

\section{More Discussions}
\label{appendix:more_discussion}
As supplementary to \Cref{sec:discussion}, here we discuss some potential extensions and impacts of our work with more details.

\paragraph{Assumption on lower-bounding elasticity as $C_{\beta}>0$.}
Here we claim that the regret lower bound should have an $\Omega(\frac{1}{C_{\beta}})$ dependence on $C_{\beta}$. We prove this by contradiction. Without loss of generality, assume $C_{\beta}\in(0,1)$. In specific, we construct a counter example to show it is impossible to have an $O(C_{\beta}^{-1+\alpha})$ regret for any $\alpha>0$:

Firstly, let $\beta=C_{\beta}$. Suppose there exists an algorithm $\cA$ that proposes a series of prices $\{p_t\}_{t=1}^{T}$ which achieve $O(C_{\beta}^{-1+\alpha})$ regret in any pricing problem instance under our assumptions.

Now, we consider another specific problem setting where $\beta=1$ while all other quantities $\theta^*, \eta^*, \{x_t\}_{t=1}^{T}$ stay unchanged. Notice that the reward function has the following property:

\begin{equation}
r(u, \beta, p) = p\cdot S(\beta p - u) = \frac1\beta\cdot(\beta p) \cdot S(\beta p - u) = \frac1\beta \cdot r(u, 1, \beta p)
\end{equation}
Therefore, we construct another algorithm $\cA^*$ which proposes $C_{\beta}\cdot p_t$ at $t=1,2,\ldots, T$. According to the $O(C_{\beta}^{-1+\alpha})$ regret bound of $\cA$, we know that $\cA^*$ will suffer $C_{\beta}\cdot O(C_{\beta}^{-1+\alpha}) = O(C_{\beta}^\alpha)$ regret. Let $C_{\beta}\rightarrow 0^+$ and observe the regret of $\cA^*$ on the latter problem setting (where $\beta=1$). On the one hand, this is a fixed problem setting with information-theoretic lower regret bound at $\Omega(\log T)$. On the other hand, the regret will be bounded by $\lim_{C_{\beta}\rightarrow 0^+} O(C_{\beta}^{\alpha}) = 0$. They are contradictory to each other. Given this, we know that there does not exist such an $\alpha>0$ such that there exists an algorithm that can achieve $O(C_{\beta}^{-1+\alpha})$. As a result, it is necessary to lower bound the elasticities by $C_{\beta}$ from $0$.

\paragraph{Adversarial attacks.}
Our PwP algorithm achieves near-optimal regret even for adversarial context sequences, while the baseline (modified) RMLP-2 algorithm fails in an oblivious adversary and suffer a linear regret. This is mainly caused by the fact that RMLP-2 takes a pure-exploration step at a \emph{fixed} time series, i.e. $t=1,1+2,1+2+3,\ldots, \frac{k(k+1)}2$. This issue might be leveraged by randomizing the position of pure-exploration steps: In each Epoch $k=1,2,\ldots$, it may firstly sample one out of all $k$ rounds in this epoch uniformly at random, and then propose a totally random price at this specific round. However, RMLP-2 still requires $\E[xx^\top]\succeq c\cdot I_d$ even with this trick.

\paragraph{Nonstationarity in Pricing}
Although our PwP algorithm is applicable on heteroscedastic valuations, we still benchmark with an optimal fixed pricing policy that knows $\eta^*$ and $\theta^*$ in advance. In reality, customers' valuations and elasticities might fluctuate according to the market environment, causing $\theta^*_t$ and $\eta^*$ different over $t\in[T]$. Existing works including \citet{leme2021learning} and \citet{baby2022non} study similar settings but assume i.i.d. noises. It is worth to further investigate the setting when heteroscedasticity and nonstationarity occur simultaneously.


\paragraph{Regret lower bounds for fixed unknown noise distributions.}
We claim a $\Omega(\sqrt{dT})$ regret lower bound in \Cref{thm:lower_bound} with customers' demand model being \Cref{equ:model}. However, this result does not imply a $\Omega(\sqrt{dT})$ regret lower bound for the contextual pricing problem with customers' valuation being $y_t=x_t^\top\theta^* + N_t$ adopted by \citet{javanmard2019dynamic, cohen2020feature_journal, xu2021logarithmic}. This is because our problem setting is more general than theirs, and our construction of $\Omega(\sqrt{dT})$ regret lower bounds are substantially beyond the scope of this specific subproblem. So far, the best existing regret lower bound for the linear noisy model ($y_t=x_t^\top\theta^* + N_t$) is still $\Omega(\sqrt{T})$. However, we conjecture that this should also be $\Omega(\sqrt{dT})$. The hardness of proving this lower bound comes from the fact that the noises are iid over time, and it is harder to be separated into several sub-sequences across $d$ that are independent to each other.

\paragraph{Algorithm and analysis for unknown link function $S(\cdot)$.}
Unfortunately, our algorithm is unable to be generalized to the online contextual pricing problem with linear valuation and unknown noise distribution that has been studied by \citet{fan2021policy}. Indeed, the problem becomes substantially harder when the noise distribution is unknown to the agent. Existing works usually adopt bandits or bandit-like algorithms to tackle that problem. For example, \citet{fan2021policy} approaches it with a combination of exploration-first and kernel method (or equivalently, local polynomial), \citet{luo2021distribution} uses a UCB-styled algorithm, and \citet{xu2022towards} adopts a discrete EXP-4 algorithm. However, none of them close the regret gap even under the homoscedastic elasticity environment as they assumed, and the known lower bound is at least $\Omega(T^{\frac23})$, or $\Omega(T^{\frac{m+1}{2m+1}})$ for smooth ones \citep{wang2021multimodal}. On the other hand, we study a parametric model, and it is not quite suitable for a bandit algorithm to achieve optimality in regret. In a nutshell, these two problems (known vs unknown noise distributions), although seem similar to each other, are indeed substantially different.

\paragraph{Linear demand model vs linear valuation model.}
In this work, we adopt a generalized linear demand model with Boolean feedback, as assumed in \Cref{equ:model}. As we have stated in \Cref{appendix:more_experiments}, there exists a heteroscedastic linear valuation model as \Cref{equ:mis_specified_model} that also captures a customer's behavior. However, this linear valuation model is actually harder to learn, as its log-likelihood function is non-convex. It is still an open problem to determine the minimax regret of an online contextual pricing problem with a valuation model like \Cref{equ:mis_specified_model}.

\paragraph{Ethic issues.}
Since we study a dynamic pricing problem, we have to consider the social impacts that our methodologies and results could have. The major concern in pricing is \emph{fairness}, which attracts increasing research interests in recent years \citep{cohen2021dynamic, cohen2022price, xu2023doubly, chen2023utility}. In general, we did not enforce or quantify the fairness of our algorithm. In fact, we might not guarantee an individual fairness since PwP proposes random prices, which means even the same input $x_t$'s would lead to different output prices. Despite the perturbations $\Delta_t$ we add to the prices, the pricing model (i.e. the parameters $\theta^*$ and $\eta^*$) is updating adaptively over time. This indicates that customers arriving later would have relatively fairer prices, since the model is evolving drastically at the beginning rounds and is converging to (local) optimal after a sufficiently long time period. We claim that our PwP algorithm is still fairer than the baseline RMLP-2 algorithm we compare with, since RMLP-2 takes pure explorations at some specific time. As a result, those customers who are given a totally random price would have a either much higher or much lower expected price than those occurring in exploitation rounds. However, it is still worth mentioning that RMLP-2 satisfies individual fairness within each pure-exploitation epoch, since it does not update parameters nor adding noises then.

\end{document}